\documentclass{article}
\usepackage[T1]{fontenc}
\usepackage{2027_conference,times}
\usepackage{amsmath,amsfonts,bm}
\usepackage{natbib}
\usepackage{url}
\usepackage{graphicx}
\usepackage{booktabs}
\usepackage{longtable}
\usepackage{float}
\usepackage{amsthm}
\usepackage{zref-savepos}
\usepackage{hyperref}
\usepackage{multirow}
\newtheorem{proposition}{Proposition}
\newcounter{algorithm}

\newcommand{\sg}{\operatorname{sg}}

\newcommand{\ours}[0]{\textsc{PivotMipl}}
\newcommand{\miplma}[0]{\textsc{MiplMa}}
\newcommand{\elimipl}[0]{\textsc{EliMipl}}
\newcommand{\demipl}[0]{\textsc{DeMipl}}
\newcommand{\miplgp}[0]{\textsc{MiplGp}}
\newcommand{\fastmipl}[0]{\textsc{FastMipl}}
\newcommand{\promipl}[0]{\textsc{ProMipl}}
\newcommand{\psmipl}[0]{\textsc{PsMipl}}
\newcommand{\dualg}[0]{\textsc{DualG}}
\newcommand{\agopmipl}[0]{\textsc{AgopMipl}}
\newcommand{\uamipl}[0]{\textsc{UAMipl}}

\title{Joint Class–-Time Learning for Video Classification with Multi-Instance Partial-Label Learning}

\author{
Lingyu Shen$^{1}$\thanks{These authors contributed equally to this work.}
\quad
Wei Tang$^{2}$\footnotemark[\value{footnote}]
\quad
Fakhri Karray$^{2}$
\quad
Min-Ling Zhang$^{1}$\thanks{Corresponding author. Email: \texttt{zhangml@seu.edu.cn}}
\\[3pt]
\normalfont\small
$^{1}$ School of Computer Science and Engineering, Southeast University, Nanjing 210096, China
\\
$^{2}$ Mohamed bin Zayed University of Artificial Intelligence (MBZUAI), Abu Dhabi, United Arab Emirates
}

\begin{document}
\maketitle

\begin{abstract}
Multi-instance partial-label learning (MIPL) addresses inexact supervision in both the instance and label spaces, which can be applied to video classification.
However, bag-level labels do not explicitly supervise the correspondence between candidate classes and temporal evidence. We propose {\ours}, which couples label disambiguation with temporal evidence allocation through a joint class--time assignment. Occupancy-regularized spherical matching associates contextualized video features while learning nonuniform temporal mass and discouraging excessive concentration. During training, candidate-restricted inference recomputes the assignment within the candidate label set. A dual-marginal KL projection then constructs a structured teacher that incorporates momentum-refined class beliefs while preserving the proposal's temporal occupancy. A single plan-level KL objective aligns the full-space predictor with this teacher. Our analysis characterizes when candidate re-solving differs from masking and shows that, under the stated construction, the joint objective decomposes into class-marginal and class-conditional temporal supervision. We construct VCMIPL benchmarks from Breakfast, DoTA, and FineAction using model-generated candidate labels and evaluate the method across four feature representations. Extensive experimental results demonstrate that {\ours} outperforms existing MIPL algorithms in both effectiveness and efficiency.
\end{abstract}

\section{Introduction}
\label{sec:introduction}
Learning to classify videos from weak annotations requires identifying both the true label and its supporting temporal evidence. Video-level labels do not specify which frames or clips support the target class within a sequence that may also contain background, transitions, and related activities, resulting in inexact supervision in the input space \citep{zhang2025rethinking}. At the class level, candidate pseudolabeling retains several plausible classes from model predictions to generate the candidate label set \citep{zhang2024candidate}. It contains the true label and several false-positive labels, which does not specify the true label and provides inexact supervision in the output space \citep{wang2025realistic}. Multi-instance partial-label learning (MIPL) captures both forms of inexact supervision by representing each example as a bag of instances paired with a candidate label set containing its true label and several false-positive labels \citep{MIPLGP}. We formulate this setting as \emph{video classification with MIPL (VCMIPL)}, where each bag is an ordered sequence of frame or clip features and its supervision is a candidate label set. 

Existing methods offer several ways to learn from inexact supervision. Attention-based MIPL couples bag representation learning with momentum-based label disambiguation \citep{DEMIPL}. In partial-label learning, prototype-guided disambiguation \citep{wang2022pico, jin2025npc} and transport-based label refinement \citep{wang2022solar, gu2025mate} exploit feature geometry and assignment constraints, respectively. Weakly-supervised video classification jointly learns classification and temporal instance selection from explicit video-level labels \citep{wang2017untrimmednets}. However, in VCMIPL, temporal evidence must be inferred while the true class remains ambiguous. Although instance selection and label disambiguation can interact during learning, refining bag-level class probabilities alone does not determine how temporal evidence is allocated to each candidate class. Distinct class–time assignments can yield identical class probabilities, even when they assign the same total mass to every temporal position. This motivates learning from a structured class–time target that combines refined class beliefs with temporal allocation inferred within the candidate set.

In this paper, we propose a VCMIPL method named {\ours}, i.e., \emph{Plan Inference via Variable-Occupancy Transport for Multi-Instance Partial-Label learning}, which represents each video through a class--prototype--time assignment. A temporal encoder produces contextualized features that are matched to multiple prototypes per class on the unit sphere. Temporal occupancy, the mass assigned to each position, is inferred jointly with the assignment, with regularization discouraging excessive concentration. During training, we re-solve the assignment within the candidate label set to obtain a proposal and refine its class marginal through momentum accumulation. A KL projection then constructs the closest teacher to this proposal subject to the refined class marginal and the proposal's temporal occupancy. A single plan-level KL objective aligns the full-space predictor with the structured teacher. At test time, prediction uses only the full-space assignment, requiring neither candidate labels nor teacher construction.

We construct three VCMIPL benchmarks based on Breakfast \citep{kuehne2014breakfast}, DoTA \citep{yao2023dota}, and FineAction \citep{liu2022fineaction} datasets, using four feature representations. Extensive experimental results on the benchmarks show that {\ours} achieves the highest mean accuracy and balanced accuracy in 10 of 12 configurations, averaging gains of $3.16\%$ and $3.26\%$ over the strongest baseline per configuration, respectively. And the accuracy is significantly higher than the strongest baseline in 8 of 12 configurations. Ablations support the contributions of occupancy regularization, adaptive temporal allocation, joint class–time supervision, dual-marginal teacher construction, multiple prototypes, and temporal encoding. Gains persist in 10 configurations against {\agopmipl} \citep{AGOPMIPL}, which is the strongest baseline, with the same temporal encoder. Further analyses show favorable accuracy–training-time trade-offs in several configurations and stable performance over the tested regularization ranges on three representative configurations.

Our contributions are threefold:
First, we introduce a structured framework for VCMIPL that couples candidate-restricted temporal allocation with label disambiguation through a joint assignment and a teacher constrained by class and temporal marginals.
Second, we analyze the inference and supervision mechanisms, characterizing the relationship between candidate re-solving and masking, the concentration control induced by occupancy regularization, and the effective supervision provided by the joint objective.
Third, we construct three VCMIPL benchmarks with model-generated candidate label sets with four feature representations, and evaluate the framework against representative MIPL methods on the benchmarks.

\section{Related work}\label{Related work}
\subsection{Multi-Instance Partial-Label Learning}
MIPL learns from bags with ambiguous candidate labels, addressing inexact supervision in both instance and label spaces. \citet{MIPLGP} introduced instance-level probabilistic disambiguation, whereas \citet{DEMIPL} used attention-based bag representations and momentum-refined label beliefs. Subsequent methods exploit complementary candidate- and non-candidate-label information~\citep{EliMIPL}, attention and prediction margins~\citep{MIPLMA}, probabilistic generative modeling of latent label distributions~\citep{ProMIPL}, and instance- and bag-level dependencies and heterogeneity for computational efficiency~\citep{FastMIPL}.
Recent approaches incorporate candidate-aware propensity scoring and deconfounding~\citep{PSMIPL}, feature and label propagation over bag- and instance-level graphs~\citep{dualg}, feature geometry based on average gradient outer products~\citep{AGOPMIPL}, aggregation uncertainty through evidential learning~\citep{UAMIPL}, and prediction calibration~\citep{tang2026calibratable}. 
However, our focus is on the specific coupling of candidate-restricted temporal occupancy, historical class beliefs, and a teacher constrained by class and temporal marginals.

\subsection{Weakly Supervised Video Learning}
Weakly supervised video learning infers temporal evidence from video-level labels. UntrimmedNets represents videos as temporal segments and jointly learns segment selection and video-level classification~\citep{wang2017untrimmednets}. Later methods model intra- and inter-segment dependencies~\citep{he2022asm}, organize evidence into proposals~\citep{ren2023proposal}, and infer salient snippet features from variations between neighboring time steps~\citep{yun2024weakly}. PseudoFormer progressively converts weak predictions into more reliable pseudo supervision for temporal localization~\citep{liu2025bridge}, while multimodal large language models provide semantic priors~\citep{zhang2025weakly}. These methods generally assume unambiguous video labels. Our setting additionally requires identifying the true class within a candidate set and linking it to supporting temporal evidence.

\section{Methodology}\label{Methodology}
\subsection{Problem Formulation and Overview}
Let $\mathcal{D}=\{(\mathbf{X}_i,S_i)\}_{i=1}^{N_{\mathrm{train}}}$ denote a MIPL training dataset of $N_{\mathrm{train}}$ videos. Each video is represented as an ordered multi-instance bag $\mathbf{X}_i=(\mathbf{x}_{i1},\ldots,\mathbf{x}_{in_i})\in\mathbb{R}^{n_i\times d}$, where $n_i$ is the number of instances and $d$ is the feature dimension. Each video has a single target label $Y_i\in\mathcal{Y}=\{1,\ldots,C\}$, but is annotated only with a nonempty candidate label set $S_i\subseteq\mathcal{Y}$ containing $Y_i$. We adopt collective bag classification, allowing multiple temporal positions to jointly support the target class without requiring every raw frame to belong to that class. 

The framework of {\ours} is shown in Figure \ref{fig:framework}. A feature-aware temporal encoder produces contextualized instance representations. Occupancy-regularized spherical transport matches these representations to multiple prototypes per class, yielding an actor plan over the full label space with adaptive temporal occupancy while discouraging excessive concentration. During training, we re-solve the transport problem within the candidate label set to obtain a proposal and refine its class marginal through momentum accumulation. We then construct a structured teacher plan that matches the refined class marginal while preserving the proposal's temporal occupancy. A plan-level KL objective aligns the actor with the detached teacher, jointly supervising class assignments and class-conditional temporal distributions.


\begin{figure}[!t]
    \centering
    \includegraphics[width=1\linewidth]{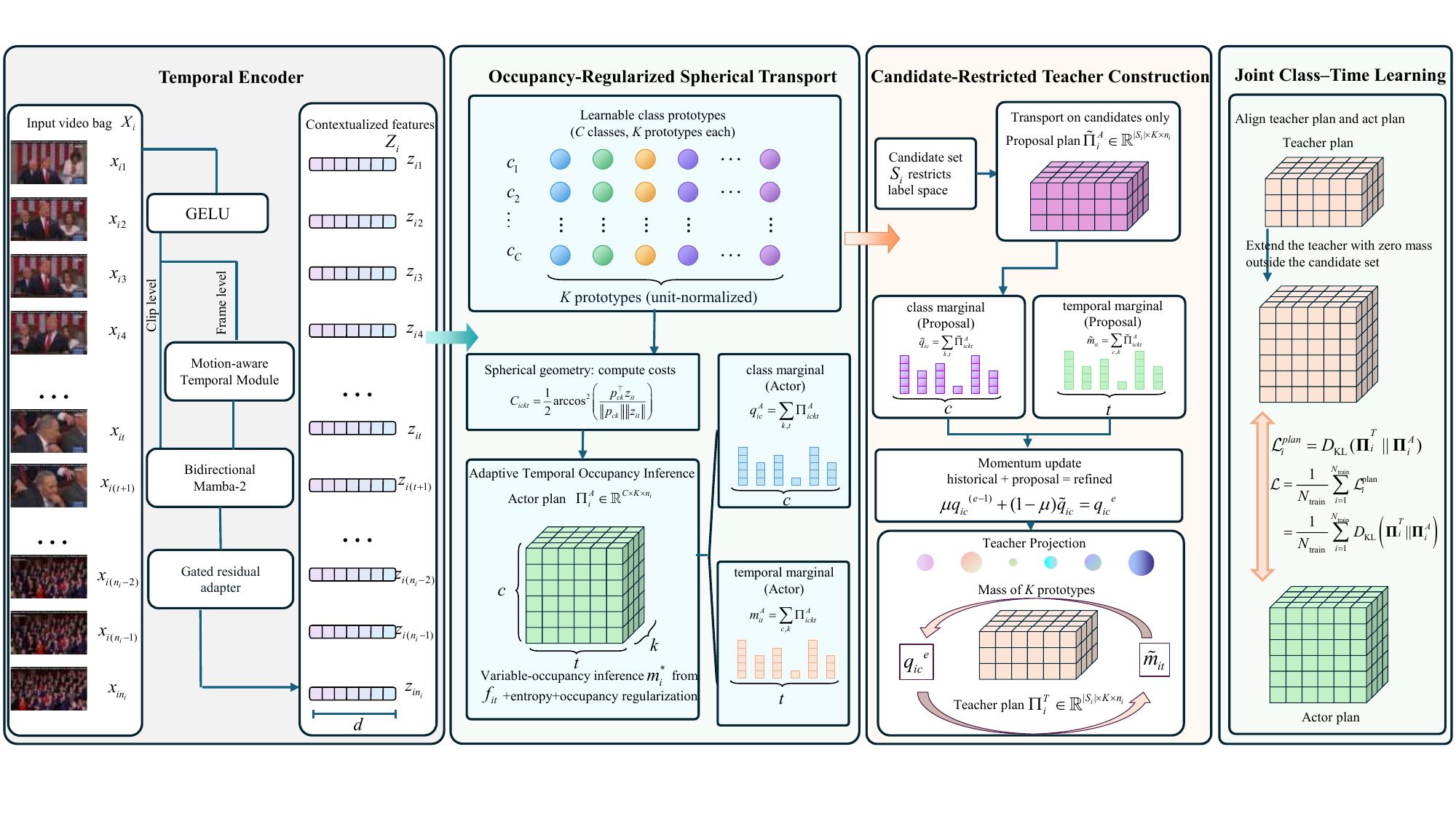}
    \vspace{-10mm}
    \caption{The framework of {\ours}}
    \label{fig:framework}
\end{figure}

\subsection{Temporal Encoder}
\label{sec:encoder}
The temporal encoder maps a multi-instance bag $\mathbf X_i\in\mathbb R^{n_i\times d}$ to contextualized features $\mathbf Z_i\in\mathbb R^{n_i\times d}$. A local difference module processes frame-level ResNet~\citep{he2016resnet} and DINOv3~\citep{simeoni2025dinov3} features, while clip-level SlowFast~\citep{feichtenhofer2019slowfast} and VideoMAEv2~\citep{wang2023videomaev2} features bypass this module because they already represent local clips. Both paths then use bidirectional Mamba-2~\citep{dao2024mamba2} and a gated residual adapter. We denote row-wise affine maps by $\operatorname{Aff}_a(H)=HW_a^\top+\mathbf1b_a^\top$, with learnable weights
and biases, and apply layer normalization (LN) over the feature dimension. The input projection maps the features from $d$-dimensional to $d_h$-dimensional channels:
\begin{equation}
 \mathbf H_i^{(0)}=\operatorname{GELU}
 \bigl(\operatorname{Aff}_{\rm in}(\operatorname{LN}(\mathbf X_i))\bigr),
 \qquad \mathbf H_i^{(0)}\in\mathbb R^{n_i\times d_h}.
 \label{eq:encoder_projection}
\end{equation}
Here $\operatorname{Aff}_{\rm in}$ maps $d$-dimensional input channels to $d_h$-dimensional hidden channels. 

\paragraph{Local Temporal Differences.}\label{Motion-Aware Module}
For frame-level features, we first smooth the projected sequence with a three-position moving average, then compute backward and forward differences:
\begin{equation}
\widehat{\mathbf h}_{it}^{(0)}=\frac{1}{3}\left(\mathbf h_{i,t-1}^{(0)}+\mathbf h_{it}^{(0)}+\mathbf h_{i,t+1}^{(0)}\right),\qquad
\Delta_{it}^{-}=\widehat{\mathbf h}_{it}^{(0)}-\widehat{\mathbf h}_{i,t-1}^{(0)},\qquad
\Delta_{it}^{+}=\widehat{\mathbf h}_{i,t+1}^{(0)}-\widehat{\mathbf h}_{it}^{(0)}.
\label{eq:local_motion}
\end{equation}
Here $\mathbf h_{it}^{(0)}$ denotes the $t$-th row of $\mathbf H_i^{(0)}$. The moving average uses replicate padding at both sequence boundaries. We set $\Delta_{i1}^{-}=\mathbf 0$ and $\Delta_{in_i}^{+}=\mathbf 0$; when $n_i=1$, both difference vectors are zero. We concatenate the two difference
sequences along the channel dimension, denoted by $[\cdot;\cdot]$, project them from $2d_h$-dimensional channels to $d_h$-dimensional ones, and apply two convolutional branches:
\begin{equation}
 \mathbf D_i=\operatorname{Aff}_{\Delta}
    ([\boldsymbol\Delta_i^-;\boldsymbol\Delta_i^+]), \qquad
 \mathbf U_i^{(j)}=\operatorname{SiLU}\bigl(\mathcal C_{\delta_j}(\mathbf D_i)\bigr),
 \qquad(\delta_1,\delta_2)=(1,2).
 \label{eq:encoder_motion_convolution}
\end{equation}
$\mathcal C_{\delta_j}$ contains a depthwise temporal convolution with kernel size of $5$ and dilation $\delta_j$, followed by a $1\times1$ pointwise convolution for channel mixing. The two branches have receptive fields of $5$ and $9$ positions over $\mathbf D_i$, respectively, and each produces an output of shape $n_i\times d_h$. Here $\operatorname{SiLU}(x)=x\sigma(x)$, where $\sigma$ is the sigmoid function. We fuse the branch outputs and form a gated motion residual: \looseness=-1
\begin{equation}
 \begin{aligned}
 \mathbf M_i&=\operatorname{SiLU}\bigl(\operatorname{LN}
       (\operatorname{Aff}_{M}([\mathbf U_i^{(1)};\mathbf U_i^{(2)}]))\bigr),\\
 \mathbf G_i&=\sigma\bigl(\operatorname{Aff}_{g}
       ([\mathbf H_i^{(0)};\mathbf M_i])\bigr),\qquad
 \mathcal M(\mathbf H_i^{(0)})=
 \mathbf H_i^{(0)}+\operatorname{Dropout}_{0.1}(\mathbf G_i\odot\mathbf M_i).
 \end{aligned}
 \label{eq:encoder_motion_gate}
\end{equation}
Both $\operatorname{Aff}_{M}$ and $\operatorname{Aff}_{g}$ map $2d_h$ channels to $d_h$. The gate $\mathbf G_i$ weights the input-dependent motion representation $\mathbf M_i$ at each temporal position and hidden channel. The input to the bidirectional encoder is $\widetilde{\mathbf H}_i^{(0)}=\mathcal M(\mathbf H_i^{(0)})$ for frame-level features and $\widetilde{\mathbf H}_i^{(0)}=\mathbf H_i^{(0)}$ for clip-level features.

\paragraph{Bidirectional Temporal Context.}
We apply $L$ bidirectional Mamba-2 blocks in sequence:
\begin{equation}
 \mathbf H_i^{(L)}=\mathcal B^{(L)}\circ\cdots\circ\mathcal B^{(1)}
       (\widetilde{\mathbf H}_i^{(0)}).
 \label{eq:encoder_blocks}
\end{equation}
Each block processes the sequence in forward and reverse temporal order. The backward output is restored to the original order, and the two outputs are combined using a feature-dependent directional gate before a feed-forward residual update. The output adapter projects the hidden features back to $d$-dimensional channels and adds a gated residual to the input:
\begin{equation}
 \mathbf Z_i=\mathbf X_i+\sigma(\gamma)\operatorname{Dropout}
 \bigl(\operatorname{Aff}_{\rm out}(\operatorname{LN}(\mathbf H_i^{(L)}))\bigr).
 \label{eq:temporal_encoder}
\end{equation}
Here $\operatorname{Aff}_{\rm out}$ maps $d_h$-dimensional channels to $d$-dimensional ones, and $\gamma\in\mathbb R$ is a learnable scalar. The identity path preserves the input features, while the residual incorporates temporal context. Therefore, the output positions are contextualized features rather than independent raw snippets.

\subsection{Occupancy-Regularized Spherical Transport}
\paragraph{Spherical Matching And Occupancy Regularization.}
Each class has $K$ learnable prototypes $\mathbf{p}_{ck}\in\mathbb{R}^d$. For nonzero features and prototypes, define the unit-normalized vectors $\bar{\mathbf z}_{it}=\mathbf z_{it}/\|\mathbf z_{it}\|_2$ and $\bar{\mathbf p}_{ck}=\mathbf p_{ck}/\|\mathbf p_{ck}\|_2$. The matching cost is half the squared geodesic distance on the unit sphere:
\begin{equation}
C_{ickt}=\tfrac12\arccos^2(\bar{\mathbf p}_{ck}^{\top}\bar{\mathbf z}_{it}).
\label{eq:spherical_cost}
\end{equation}
For any nonempty admissible class set $\mathcal A \subseteq\mathcal Y$, define $\Delta_i^{\mathcal A}$ as the unit simplex over $(c,k,t)\in\mathcal A\times\{1,\ldots,K\}\times\{1,\ldots,n_i\}$. With temporal occupancy $m_{it}=\sum_{c\in\mathcal A,k}\Pi_{ickt}$, our inference operator is \looseness=-1
\begin{equation}
\begin{split}
\mathcal T_{\varepsilon,\lambda}(\mathbf C_i;\mathcal A)
=\arg\min_{\boldsymbol\Pi\in\Delta_i^{\mathcal A}}\bigg\{
&\sum_{c\in\mathcal A,k,t}\Pi_{ickt}C_{ickt}
+\varepsilon\sum_{c\in\mathcal A,k,t}\Pi_{ickt}(\log\Pi_{ickt}-1)\\
&+\frac\lambda2\left(n_i\sum_t m_{it}^2-1\right)\bigg\},
\end{split}\label{eq:occupancy_transport}
\end{equation}
where $\varepsilon>0$ and $\lambda\geq0$. Here, ``transport'' refers to a regularized assignment with freely optimized class and temporal marginals.
Let $\mathbf u_i$ be the uniform temporal distribution, with $u_{it}=1/n_i$. The occupancy penalty can be written as 
\begin{equation}
R_{\rm occ}(\mathbf m_i)=\chi^2(\mathbf m_i\|\mathbf u_i)
=e^{D_2(\mathbf m_i\|\mathbf u_i)}-1
=n_i\sum_t m_{it}^2-1,
\label{eq:occupancy_regularizer}
\end{equation}
where $\chi^2$ and $D_2$ denote the Pearson $\chi^2$ divergence and the order-2 R\'enyi divergence, respectively. This penalty discourages excessive mass on a few temporal positions without enforcing uniform occupancy. It controls concentration but does not identify action locations. More information about the concentration properties can be found in Appendix~\ref{app:revth:support}.

\paragraph{Factorized Inference.}
For fixed temporal occupancy, the conditional minimization over classes and prototypes admits a Gibbs solution. Define
\begin{equation}
\mathcal Z_{it}^{\mathcal A}=\sum_{c\in\mathcal A,k}e^{-C_{ickt}/\varepsilon},\qquad
\rho_{ickt}^{\mathcal A}=\frac{e^{-C_{ickt}/\varepsilon}}{\mathcal Z_{it}^{\mathcal A}},\qquad
f_{it}^{\mathcal A}=-\varepsilon\log\mathcal Z_{it}^{\mathcal A}.
\label{eq:gibbs_factorization}
\end{equation}
Substituting $\Pi_{ickt}=m_{it}\rho_{ickt}^{\mathcal A}$ into Eq.~(\ref{eq:occupancy_transport}) and dropping terms independent of $\mathbf m_i$ gives
\begin{equation}
\mathbf m_i^\star=\arg\min_{\mathbf m_i\in\Delta_{n_i}}
\left\{\langle\mathbf m_i,\mathbf f_i^{\mathcal A}\rangle
+\varepsilon D_{\rm KL}(\mathbf m_i\|\mathbf u_i)
+\frac\lambda2\left(n_i\|\mathbf m_i\|_2^2-1\right)\right\},
\label{eq:time_occupancy_problem}
\end{equation}
where $\Delta_{n_i}$ is the unit simplex in $\mathbb{R}^{n_i}$. The remaining problem is to optimize an $n_i$-dimensional occupancy vector, with its normalization multiplier determined via scalar root-finding. The optimal plan, full-space actor plan, and actor class marginal are, respectively,
\begin{equation}
\Pi_{ickt}^\star=m_{it}^\star\rho_{ickt}^{\mathcal A},\qquad
\boldsymbol\Pi_i^A=\mathcal T_{\varepsilon,\lambda}(\mathbf C_i;\mathcal Y),\qquad
q_{ic}^A=\sum_{k,t}\Pi_{ickt}^A.
\label{eq:final_transport_plan}
\end{equation}

\paragraph{Occupancy Solver.}
For $\lambda>0$, the occupancy can be expressed via the normalization multiplier $\xi_i$:  \looseness=-1
\begin{equation}
 m_{it}(\xi_i)=\frac{\varepsilon}{\lambda n_i}
 W\!\left(\frac{\lambda}{\varepsilon}
 e^{(\xi_i-f_{it}^{\mathcal A})/\varepsilon}\right),\qquad
 g_i(\xi_i)=\sum_t m_{it}(\xi_i)-1,
\label{eq:occupancy_lambert}
\end{equation}
where $W$ is the principal Lambert-$W$ function, defined by $W(x)e^{W(x)}=x$. We choose the multiplier to satisfy $g_i(\xi_i)=0$. The described solver initializes $\xi_i^{(0)}=-\varepsilon\log\sum_t u_{it}\exp(-f_{it}^{\mathcal A}/\varepsilon)$ and applies Newton updates:
\begin{equation}
    \xi_i^{(r+1)}=\xi_i^{(r)} - \frac{g_i(\xi_i^{(r)})}{\displaystyle\sum_t\frac{m_{it}(\xi_i^{(r)})}{\varepsilon+\lambda n_i m_{it}(\xi_i^{(r)})}}.
\label{eq:occupancy_newton}
\end{equation}
When $\lambda=0$, the solution instead takes the form
$m_{it}^\star=\exp(-f_{it}^{\mathcal A}/\varepsilon)/\sum_s\exp(-f_{is}^{\mathcal A}/\varepsilon)$.
This case requires a separate branch as Eq.~(\ref{eq:occupancy_lambert}) contains $1/\lambda$, and entropy regularization remains active. Appendix~\ref{app:revth:occupancy} derives these expressions and establishes the uniqueness of the positive optimum.

We compute $f_{it}^{\mathcal A}$ using log-sum-exp and evaluate the conditional assignments in the log domain. For $\lambda>0$, the Lambert-$W$ term is evaluated directly from its log argument: letting $h_{it}=\log(\lambda/\varepsilon)+(\xi_i-f_{it}^{\mathcal A})/\varepsilon$, we solve $w_{it}+\log w_{it}=h_{it}$ for $w_{it}=W(e^{h_{it}})$ by Newton iteration, avoiding explicit construction of $e^{h_{it}}$. We then update the scalar normalization multiplier by Newton iteration and recover $m_{it}=\varepsilon w_{it}/(\lambda n_i)$. The occupancy solve is performed in double precision, and we monitor both its normalization and KKT stationarity residuals. For $\lambda=0$, we directly use the closed-form softmax solution.

\paragraph{Class Belief Update.}\label{sec:proposal}
During training, we construct a proposal on the candidate label set $S_i$ using the full-space actor's current costs and temperature: 
\begin{equation}
\widetilde{\boldsymbol\Pi}_i=\mathcal T_{\varepsilon,\lambda}(\sg(\mathbf C_i);S_i),\qquad
\widetilde q_{ic}=\sum_{k,t}\widetilde\Pi_{ickt},\qquad
\widetilde m_{it}=\sum_{c\in S_i,k}\widetilde\Pi_{ickt}.
\label{eq:proposal_transport}
\end{equation}
Here, $\sg$ denotes stop-gradient. We recompute the partition functions and free energies over candidate classes and apply the same occupancy solver to obtain $\widetilde{\mathbf m}_i$. Therefore, this restriction determines class support and can change temporal allocation.

Re-solving generally differs from globally masking and renormalizing the actor. Global masking also changes temporal mass, but it need not optimize the restricted occupancy-regularized objective. Appendix~\ref{app:revth:mask} gives the precise equivalence conditions, including the $\lambda=0$ boundary. The method uses re-solving rather than a fixed-time conditional renormalization. Initialize $q_{ic}^{(0)}=\mathbf1[c\in S_i]/|S_i|$. At epoch $e$, combine the stored belief with the current proposal marginal: 
\begin{equation}
q_{ic}^{(e)}=\mu q_{ic}^{(e-1)}+(1-\mu)\widetilde q_{ic},\qquad \mu\in[0,1),\ c\in S_i,
\label{eq:plan_momentum_update}
\end{equation}
and set $q_{ic}^{(e)}=0$ for $c\notin S_i$. The EMA averages class marginals, not full plans or network parameters. Each training bag is visited once per epoch. At its visit in epoch $e$, the stored $q_i^{(e-1)}$ is combined with the current detached proposal marginal to form $q_i^{(e)}$; this refined belief is used immediately for teacher construction and is stored after the optimizer step for the next epoch. 

\subsection{Candidate-Restricted Teacher Construction}\label{sec:teacher}
The teacher incorporates historical class mass while retaining the proposal's current temporal occupancy. We obtain it through the KL projection on the candidate support:
\begin{equation}
\begin{aligned}
\boldsymbol\Pi_i^T&=\arg\min_{\boldsymbol\Pi\geq0}
D_{\rm KL}(\boldsymbol\Pi\|\widetilde{\boldsymbol\Pi}_i),\\
\text{subject to}\quad
\sum_{k,t}\Pi_{ickt}&=q_{ic}^{(e)}\quad(c\in S_i),\qquad
\sum_{c\in S_i,k}\Pi_{ickt}=\widetilde m_{it}\quad(t=1,\ldots,n_i).
\end{aligned}\label{eq:teacher_projection}
\end{equation}
Let $\widetilde M_{ict}=\sum_k\widetilde\Pi_{ickt}$ denote the proposal's class--time matrix. The teacher has the multiplicative form $\Pi_{ickt}^T=\alpha_{ic}\widetilde\Pi_{ickt}\beta_{it}$. Starting from positive scaling factors, we alternate the updates
\begin{equation}
 \alpha_{ic}\leftarrow\frac{q_{ic}^{(e)}}{\sum_t\widetilde M_{ict}\beta_{it}},\qquad
 \beta_{it}\leftarrow\frac{\widetilde m_{it}}{\sum_{c\in S_i}\widetilde M_{ict}\alpha_{ic}}.
 \label{eq:teacher_scaling}
\end{equation}
Each update enforces one target marginal; both constraints hold at convergence.
We perform the marginal scaling in the log domain using log-sum-exp denominators.
Computing the scaling factors requires only the $|S_i|\times n_i$ class--time
matrix. The full teacher is then recovered by multiplying each proposal entry
by $\alpha_{ic}\beta_{it}$, with the same factors applied across prototypes.

We extend the teacher with zeros outside $S_i$ to obtain $\widehat{\boldsymbol{\Pi}}_i^T$. The training objective is
\begin{equation}
\mathcal L=\frac{1}{N_{\mathrm{train}}}\sum_i\mathcal L_i^{\mathrm{plan}},\qquad
\mathcal L_i^{\mathrm{plan}}=
D_{\mathrm{KL}}\!\left(
\operatorname{sg}(\widehat{\boldsymbol{\Pi}}_i^T)
\Vert
\boldsymbol{\Pi}_i^A
\right).
\label{eq:plan_loss}
\end{equation}
Under the same-cost construction and exact teacher projection, the prototype-conditional distributions of the actor and teacher coincide on the candidate support. Therefore, the plan-level objective reduces to
\begin{equation}
\mathcal L_i^{\mathrm{plan}}
=
D_{\mathrm{KL}}(\mathbf q_i^T\Vert\mathbf q_i^A)
+
\sum\nolimits_{c:q_{ic}^T>0}q_{ic}^T
D_{\mathrm{KL}}\!\left(
P_i^T(t\mid c)\Vert P_i^A(t\mid c)
\right),
\label{eq:plan_decomposition}
\end{equation}
where $q_{ic}^{T}$ and $q_{ic}^{A}$ are the teacher and actor class marginals, respectively, and $P_i^{T}(t\mid c)$ and $P_i^{A}(t\mid c)$ denote their class-conditional temporal distributions obtained after marginalizing over the prototype dimension. Thus, beyond class-marginal alignment, the joint objective explicitly supervises the class-conditional temporal allocation. A complete derivation and the corresponding current-point gradient equivalence are provided in Appendix~\ref{app:revth:kl}.

The loss uses the full-space actor without renormalizing it over $S_i$, thereby penalizing mass assigned to non-candidate classes. Detaching the entire target-construction pipeline gives
\begin{equation}
 \nabla_\theta\mathcal L_i^{\rm plan}
 =-\sum\nolimits_{c\in S_i,k,t}\Pi_{ickt}^T\nabla_\theta\log\Pi_{ickt}^A.
 \label{eq:plan_gradient}
\end{equation}
The actor gradient propagates through both temporal occupancy and Gibbs conditional assignments to the matching costs, prototypes, and temporal encoder. We differentiate through the unrolled actor solver, and no gradients pass through proposal inference, EMA updates, or teacher projection.

The resulting supervision acts on class--time associations parameterized by multiple prototypes. Under the stated same-cost construction, the actor and teacher have identical conditional distributions over prototypes given class and time on the candidate support. The prototype-conditional KL term is therefore zero, although prototypes still affect the class--time distribution and can receive gradients through it. The full KL decomposition, the corresponding gradient equivalence at the teacher construction point, and boundary cases are established in Appendix~\ref{app:revth:kl}--\ref{app:revth:degeneracy}.

\subsection{Joint Class–Time Learning and Prediction}
Algorithm~\ref{alg:training} summarizes the per-bag training procedure. Training uses one complete video bag per optimization step: we compute one plan-level loss, perform one backward pass and gradient-clipped optimizer update, and then write the detached refined class belief back to the training dataset. Since bags are processed individually, no temporal padding is introduced. At test time, we compute only $\mathbf Z_i$, the matching costs, and the full-space actor, and predict $\hat y_i=\arg\max_{c\in\mathcal Y}\sum_{k,t}\Pi^A_{ickt}$.

\begin{figure}[t]
\centering
\begin{minipage}{\linewidth}
\hrule\smallskip
\refstepcounter{algorithm}\textbf{Algorithm \thealgorithm: {\ours} training and prediction}\label{alg:training}\par\smallskip
\begin{tabular}{@{}r@{\quad}p{0.92\linewidth}@{}}
1 & Initialize parameters and $q_{ic}^{(0)}=\mathbf1[c\in S_i]/|S_i|$.\\
2 & For each epoch $e$ and each training bag $(X_i,S_i)$:\\
3 & \quad Encode valid ordered positions and compute $\mathbf C_i$.\\
4 & \quad Solve full-space actor $\boldsymbol\Pi_i^A=\mathcal T(\mathbf C_i;\mathcal Y)$.\\
5 & \quad Solve proposal $\widetilde{\boldsymbol\Pi}_i=\mathcal T(\sg(\mathbf C_i);S_i)$.\\
6 & \quad Update class belief using Eq.~(\ref{eq:plan_momentum_update}).\\
7 & \quad Scale the proposal to $(\mathbf q_i^{(e)},\widetilde{\mathbf m}_i)$ using Eq.~(\ref{eq:teacher_scaling}).\\
8 & \quad Extend teacher by zero outside $S_i$ and detach it.\\
9 & \quad Backpropagate the actor KL, clip the global gradient norm, step the optimizer, and store the detached refined belief.\\
10 & At test time, return $\arg\max_c\sum_{k,t}\Pi_{ickt}^A$ without candidates.\\
\end{tabular}\par\smallskip
\smallskip\hrule
\end{minipage}
\vspace{-10pt}
\end{figure}

\section{Experiments}\label{Experiments}
\subsection{Experimental Setup}
\paragraph{Benchmark Construction}
We construct three VCMIPL benchmarks, i.e., Breakfast-MIPL, DoTA-MIPL, and FineAction-MIPL, from the Breakfast~\citep{kuehne2014breakfast}, DoTA~\citep{yao2023dota}, and FineAction~\citep{liu2022fineaction} datasets, respectively. For FineAction, we retain single-label videos and exclude classes with fewer than 50 videos. The resulting subsets contain 1,989 Breakfast videos (10 classes), 4,577 DoTA videos (9 classes), and 10,675 FineAction videos (78 classes). Appendix~\ref{app:protocol} reports candidate-set and feature statistics. We use ResNet~\citep{he2016resnet} and DINOv3~\citep{simeoni2025dinov3} for frame-level features, and SlowFast~\citep{feichtenhofer2019slowfast} and VideoMAEv2~\citep{wang2023videomaev2} for clip-level features.
To reduce the annotation cost, we generate candidate labels using a VideoMAEv2 recognizer equipped with a TimeSformerHead in MMAction2~\citep{2020mmaction2}. We fine-tune the recognizer on four folds and generate predictions for the held-out fold, repeating this procedure for all five folds. Fine-tuning uses 30 epochs on DoTA and 100 epochs on both Breakfast and FineAction. Following \citet{zhang2024candidate}, we first select labels for each video in descending order of confidence until their cumulative confidence reaches $95\%$. A selected label is retained only if the video also ranks among the top $50\%$ of videos by confidence for that label.

\paragraph{Comparative Methods}
We compare {\ours} with seven MIPL baselines: \demipl~\citep{DEMIPL}, \elimipl~\citep{EliMIPL}, \miplma~\citep{MIPLMA}, \promipl~\citep{ProMIPL}, \psmipl~\citep{PSMIPL}, \dualg~\citep{dualg}, and \agopmipl~\citep{AGOPMIPL}. We exclude \miplgp~\citep{MIPLGP}, \fastmipl~\citep{FastMIPL}, and {\uamipl}~\citep{UAMIPL} because their released implementations require all instance-level features to be loaded into memory, making them impractical for these video benchmarks under the available memory budget.

\paragraph{Implementation Details}
{\ours} is implemented in PyTorch and trained for 40 epochs on an NVIDIA H100 GPU using AdamW with $(\beta_1,\beta_2)=(0.9,0.999)$ and zero weight decay. The learning rates are $5\times10^{-3}$ for the prototypes and $5\times10^{-4}$ for the temporal encoder. We evaluate {\ours} over the same five predefined folds used for cross-fitting the candidate-label recognizer. We report classification accuracy (ACC) and balanced accuracy (BACC) as mean $\pm$ standard deviation over five folds, where BACC is the mean recall over ground-truth classes represented in each test fold. More implementation details are provided in Appendix~\ref{app:implementation}. 

\begin{table}[!t]
\centering
\vspace{-10pt}
\caption{Classification accuracy on the benchmark datasets. Results are reported as (mean $\pm$ std \%). The best result in each column is bolded, and the second-best results are underlined. $\bullet$/$\circ$ indicates whether {\ours} is statistically superior/inferior to the corresponding method under a two-sided paired $t$-test at $p<0.05$.}
\label{tab:acc-comparison-transposed}
\small
\setlength{\tabcolsep}{8pt}
\begin{tabular}{llcccc}
\toprule
Dataset & Method & DINOv3 & VideoMAEv2 & ResNet & SlowFast\\
\midrule
\multirow{8}{*}{\textbf{\shortstack{Breakfast\\-MIPL}}}
& \demipl & $11.4\%\,{\scriptstyle\pm\,0.4\%}\,\bullet$ & $25.7\%\,{\scriptstyle\pm\,10.0\%}\,\bullet$ & $10.1\%\,{\scriptstyle\pm\,0.7\%}\,\bullet$ & $11.0\%\,{\scriptstyle\pm\,0.5\%}\,\bullet$\\
& \elimipl & $27.8\%\,{\scriptstyle\pm\,0.8\%}\,\bullet$ & $52.2\%\,{\scriptstyle\pm\,1.9\%}\,\bullet$ & $26.1\%\,{\scriptstyle\pm\,1.7\%}\,\bullet$ & $11.3\%\,{\scriptstyle\pm\,0.1\%}\,\bullet$\\
& \miplma & $33.9\%\,{\scriptstyle\pm\,1.5\%}\,\bullet$ & $63.6\%\,{\scriptstyle\pm\,2.1\%}\,\bullet$ & $30.5\%\,{\scriptstyle\pm\,2.4\%}\,\bullet$ & $11.3\%\,{\scriptstyle\pm\,0.1\%}\,\bullet$\\
& \promipl & $\underline{62.2\%}\,{\scriptstyle\pm\,1.2\%}\,\bullet$ & $83.8\%\,{\scriptstyle\pm\,3.4\%}\,\bullet$ & $\underline{43.7\%}\,{\scriptstyle\pm\,1.2\%}\,\bullet$ & $41.6\%\,{\scriptstyle\pm\,1.1\%}\,\bullet$\\
& \psmipl & $47.3\%\,{\scriptstyle\pm\,5.0\%}\,\bullet$ & $70.3\%\,{\scriptstyle\pm\,7.3\%}\,\bullet$ & $11.0\%\,{\scriptstyle\pm\,1.4\%}\,\bullet$ & $33.0\%\,{\scriptstyle\pm\,0.6\%}\,\bullet$\\
& \dualg & $45.7\%\,{\scriptstyle\pm\,3.2\%}\,\bullet$ & $64.9\%\,{\scriptstyle\pm\,4.4\%}\,\bullet$ & $34.3\%\,{\scriptstyle\pm\,2.0\%}\,\bullet$ & $14.4\%\,{\scriptstyle\pm\,1.6\%}\,\bullet$\\
& \agopmipl & $46.3\%\,{\scriptstyle\pm\,4.2\%}\,\bullet$ & $\mathbf{87.6\%}\,{\scriptstyle\pm\,4.3\%}$ & $33.3\%\,{\scriptstyle\pm\,3.1\%}\,\bullet$ & $\underline{50.4\%}\,{\scriptstyle\pm\,2.5\%}\,\bullet$\\
& {\ours} & $\mathbf{66.3\%}\,{\scriptstyle\pm\,1.0\%}$ & $\underline{86.1\%}\,{\scriptstyle\pm\,4.5\%}$ & $\mathbf{52.3\%}\,{\scriptstyle\pm\,3.4\%}$ & $\mathbf{55.9\%}\,{\scriptstyle\pm\,1.8\%}$\\
\midrule
\multirow{8}{*}{\textbf{\shortstack{DoTA\\-MIPL}}}
& \demipl & $45.6\%\,{\scriptstyle\pm\,1.3\%}\,\bullet$ & $68.1\%\,{\scriptstyle\pm\,0.7\%}\,\bullet$ & $42.6\%\,{\scriptstyle\pm\,0.4\%}\,\bullet$ & $41.2\%\,{\scriptstyle\pm\,1.6\%}\,\bullet$\\
& \elimipl & $45.8\%\,{\scriptstyle\pm\,1.6\%}\,\bullet$ & $68.3\%\,{\scriptstyle\pm\,1.0\%}\,\bullet$ & $43.4\%\,{\scriptstyle\pm\,0.5\%}\,\bullet$ & $41.7\%\,{\scriptstyle\pm\,1.0\%}\,\bullet$\\
& \miplma & $46.8\%\,{\scriptstyle\pm\,1.1\%}\,\bullet$ & $69.2\%\,{\scriptstyle\pm\,0.9\%}\,\bullet$ & $43.7\%\,{\scriptstyle\pm\,1.0\%}\,\bullet$ & $41.8\%\,{\scriptstyle\pm\,1.1\%}\,\bullet$\\
& \promipl & $49.2\%\,{\scriptstyle\pm\,1.8\%}\,\bullet$ & $72.5\%\,{\scriptstyle\pm\,1.5\%}$ & $\underline{45.9\%}\,{\scriptstyle\pm\,2.2\%}\,\bullet$ & $\underline{48.0\%}\,{\scriptstyle\pm\,0.4\%}$\\
& \psmipl & $\underline{49.8\%}\,{\scriptstyle\pm\,0.9\%}\,\bullet$ & $\underline{74.1\%}\,{\scriptstyle\pm\,0.9\%}$ & $44.6\%\,{\scriptstyle\pm\,1.3\%}\,\bullet$ & $26.0\%\,{\scriptstyle\pm\,13.0\%}\,\bullet$\\
& \dualg & $46.1\%\,{\scriptstyle\pm\,2.6\%}\,\bullet$ & $69.6\%\,{\scriptstyle\pm\,0.8\%}\,\bullet$ & $44.3\%\,{\scriptstyle\pm\,1.2\%}\,\bullet$ & $43.6\%\,{\scriptstyle\pm\,0.9\%}\,\bullet$\\
& \agopmipl & $49.7\%\,{\scriptstyle\pm\,1.3\%}$ & $69.9\%\,{\scriptstyle\pm\,1.8\%}\,\bullet$ & $45.2\%\,{\scriptstyle\pm\,0.7\%}\,\bullet$ & $47.0\%\,{\scriptstyle\pm\,0.1\%}\,\bullet$\\
& {\ours} & $\mathbf{52.7\%}\,{\scriptstyle\pm\,2.0\%}$ & $\mathbf{74.7\%}\,{\scriptstyle\pm\,1.8\%}$ & $\mathbf{49.6\%}\,{\scriptstyle\pm\,1.2\%}$ & $\mathbf{48.8\%}\,{\scriptstyle\pm\,0.6\%}$\\
\midrule
\multirow{8}{*}{\textbf{\shortstack{FineAction\\-MIPL}}}
& \demipl & $41.5\%\,{\scriptstyle\pm\,1.6\%}\,\bullet$ & $52.1\%\,{\scriptstyle\pm\,1.6\%}\,\bullet$ & $21.7\%\,{\scriptstyle\pm\,1.8\%}\,\bullet$ & $50.2\%\,{\scriptstyle\pm\,2.1\%}\,\bullet$\\
& \elimipl & $42.0\%\,{\scriptstyle\pm\,1.0\%}\,\bullet$ & $49.6\%\,{\scriptstyle\pm\,2.7\%}\,\bullet$ & $22.9\%\,{\scriptstyle\pm\,2.0\%}\,\bullet$ & $49.3\%\,{\scriptstyle\pm\,1.7\%}\,\bullet$\\
& \miplma & $55.4\%\,{\scriptstyle\pm\,1.4\%}\,\bullet$ & $68.3\%\,{\scriptstyle\pm\,0.5\%}\,\bullet$ & $35.7\%\,{\scriptstyle\pm\,3.0\%}\,\bullet$ & $57.2\%\,{\scriptstyle\pm\,2.3\%}\,\bullet$\\
& \promipl & $80.7\%\,{\scriptstyle\pm\,1.5\%}\,\bullet$ & $92.0\%\,{\scriptstyle\pm\,3.4\%}$ & $69.7\%\,{\scriptstyle\pm\,1.4\%}\,\bullet$ & $82.4\%\,{\scriptstyle\pm\,2.2\%}\,\bullet$\\
& \psmipl & $60.5\%\,{\scriptstyle\pm\,1.4\%}\,\bullet$ & $66.2\%\,{\scriptstyle\pm\,1.5\%}\,\bullet$ & $43.0\%\,{\scriptstyle\pm\,2.0\%}\,\bullet$ & $32.7\%\,{\scriptstyle\pm\,14.9\%}\,\bullet$\\
& \dualg & $58.1\%\,{\scriptstyle\pm\,1.4\%}\,\bullet$ & $74.2\%\,{\scriptstyle\pm\,0.8\%}\,\bullet$ & $46.4\%\,{\scriptstyle\pm\,1.3\%}\,\bullet$ & $63.9\%\,{\scriptstyle\pm\,0.8\%}\,\bullet$\\
& \agopmipl & $\underline{82.2\%}\,{\scriptstyle\pm\,1.3\%}\,\bullet$ & $\mathbf{96.3\%}\,{\scriptstyle\pm\,0.5\%}\,\circ$ & $\underline{71.1\%}\,{\scriptstyle\pm\,1.0\%}\,\bullet$ & $\underline{84.2\%}\,{\scriptstyle\pm\,2.2\%}\,\bullet$\\
& {\ours} & $\mathbf{88.5\%}\,{\scriptstyle\pm\,1.0\%}$ & $\underline{94.0\%}\,{\scriptstyle\pm\,1.4\%}$ & $\mathbf{76.2\%}\,{\scriptstyle\pm\,1.2\%}$ & $\mathbf{88.3\%}\,{\scriptstyle\pm\,1.0\%}$\\
\bottomrule
\end{tabular}
\end{table}

\subsection{Main Experimental Results}
Table~\ref{tab:acc-comparison-transposed} shows that {\ours} achieves the highest mean ACC in 10 of the 12 dataset--feature configurations, with an average gain of 3.16\% over the strongest baseline in each configuration. Beyond the overall improvement, the performance pattern supports the central motivation of our framework. Most competing MIPL methods ultimately resolve ambiguity through bag-level class supervision, without explicitly retaining which temporal evidence supports each candidate class. In contrast, {\ours} preserves this class--time association throughout disambiguation, so that the plan-level objective constrains both the class belief and its supporting temporal allocation. This distinction is particularly evident with the frame-level DINOv3 and ResNet representations, where {\ours} consistently ranks first across all three datasets. These representations leave temporal dependencies largely unresolved before MIPL learning, making explicit evidence allocation especially useful for separating informative instances from background or class-irrelevant content. The main exceptions occur in VideoMAEv2, where strong baselines already achieve highly competitive performance. A plausible explanation is that its clip-level representations already encode substantial local spatiotemporal context, reducing the remaining ambiguity that can be resolved by downstream temporal restructuring. Nevertheless, {\ours} is inferior to {\agopmipl} in only one configuration while being significantly superior to the strongest competing method in eight, suggesting that its advantage primarily stems from coupling label disambiguation with temporal evidence allocation rather than from a particular feature backbone.

\vspace{-5pt}
\subsection{Ablation Study}\label{app:ablation}
We evaluate six ablation variants across all 12 dataset--feature configurations using Fold~1 ACC at epoch~40. Table~\ref{tab:pivotmipl-ablation-fold1} presents the individual results. Moreover, a summary of the ablations across twelve dataset--feature configurations is provided in Appendix~\ref{Ablation summaries}. \textbf{A1} removes occupancy regularization by setting $\lambda=0$, while retaining entropy regularization and learned temporal occupancy. \textbf{A2} fixes temporal occupancy to the uniform distribution, $m_{it}=1/n_i$. \textbf{A3} replaces the joint KL objective with a class-level KL objective. \textbf{A4} replaces the dual-marginal teacher projection with class-marginal-only rescaling, removing the constraint that preserves the proposal's temporal occupancy. \textbf{A5} uses a single prototype per class, i.e., $K=1$. \textbf{A6} removes the temporal encoder.

Table~\ref{tab:pivotmipl-ablation-fold1} shows that every ablation reduces mean ACC across the 12 configurations. Notably, A1 and A2 reduce ACC in 10 and 9 configurations, favoring regularized yet adaptive temporal allocation. A3 and A4 degrade performance across all configurations, even with explicit class--time supervision and occupancy-preserving teacher construction. A5 and A6 each reduce ACC in 11 configurations, supporting multiple prototypes and temporal contextualization. Temporal encoding has the largest average impact, while the joint objective and dual-marginal teacher deliver the most consistent gains. 

Appendix~\ref{BACC} reports the BACC results, with {\ours} achieving the highest mean performance in $10$ of the $12$ configurations. To assess the role of temporal encoding, In Appendix~\ref{te}, we compare {\ours} with {\agopmipl}, for which we adopt the same temporal encoder as {\ours}. {\ours} outperforms this baseline in $10$ of the $12$ configurations, suggesting that the encoder alone does not account for the observed gains. Appendix~\ref{Efficiency analysis} reports favorable trade-offs between accuracy and computational efficiency across most evaluated settings. Appendix~\ref{sensitivity analysis} shows limited performance variation over the tested ranges of the main transport hyperparameters.

\begin{table}[!t]
\centering
\vspace{-10pt}
\caption{Ablation study of \ours.}
\label{tab:pivotmipl-ablation-fold1}
\setlength{\tabcolsep}{4pt}   
\begin{tabular}{llccccccc}
\toprule
Dataset & Feature & {\ours} & A1 & A2 & A3 & A4 & A5 & A6\\
\midrule
\multirow{4}{*}{\textbf{\shortstack{Breakfast\\-MIPL}}}
& DINOv3 & $\underline{66.4\%}$ & $62.6\%$ & $\mathbf{72.1\%}$ & $66.1\%$ & $64.8\%$ & $64.3\%$ & $57.1\%$\\
& VideoMAEv2 & $\underline{77.4\%}$ & $\mathbf{77.8\%}$ & $77.1\%$ & $70.3\%$ & $67.1\%$ & $72.6\%$ & $75.6\%$\\
& ResNet & $\underline{52.2\%}$ & $42.6\%$ & $\mathbf{54.4\%}$ & $51.4\%$ & $50.3\%$ & $47.9\%$ & $35.4\%$\\
& SlowFast & $\mathbf{53.9\%}$ & $48.4\%$ & $53.4\%$ & $52.6\%$ & $51.6\%$ & $\underline{53.6\%}$ & $49.9\%$\\
\midrule
\multirow{4}{*}{\textbf{\shortstack{DoTA\\-MIPL}}}
& DINOv3 & $\mathbf{54.3\%}$ & $50.3\%$ & $51.7\%$ & $\underline{53.5\%}$ & $52.6\%$ & $51.3\%$ & $49.3\%$\\
& VideoMAEv2 & $\underline{72.8\%}$ & $72.0\%$ & $70.9\%$ & $72.2\%$ & $72.7\%$ & $\mathbf{73.0\%}$ & $72.7\%$\\
& ResNet & $\mathbf{50.9\%}$ & $48.1\%$ & $49.8\%$ & $\underline{50.5\%}$ & $50.4\%$ & $47.2\%$ & $47.5\%$\\
& SlowFast & $\mathbf{48.4\%}$ & $46.4\%$ & $\underline{47.8\%}$ & $45.9\%$ & $47.1\%$ & $46.5\%$ & $43.9\%$\\
\midrule
\multirow{4}{*}{\textbf{\shortstack{FineAction\\-MIPL}}}
& DINOv3 & $\mathbf{90.0\%}$ & $\underline{89.7\%}$ & $86.3\%$ & $79.0\%$ & $86.3\%$ & $85.7\%$ & $89.4\%$\\
& VideoMAEv2 & $92.4\%$ & $93.7\%$ & $\underline{96.3\%}$ & $88.7\%$ & $89.1\%$ & $91.1\%$ & $\mathbf{96.7\%}$\\
& ResNet & $\mathbf{78.1\%}$ & $75.5\%$ & $71.9\%$ & $72.6\%$ & $\underline{77.6\%}$ & $75.4\%$ & $75.8\%$\\
& SlowFast & $\mathbf{90.0\%}$ & $80.7\%$ & $86.5\%$ & $84.9\%$ & $86.9\%$ & $\underline{88.0\%}$ & $64.3\%$\\
\bottomrule
\end{tabular}
\end{table}

\vspace{-5pt}
\section{Conclusion}\label{Conclusion}
We presented {\ours}, a structured framework coupling candidate-label disambiguation with temporal evidence allocation for video multi-instance partial-label learning. Occupancy-regularized spherical transport matches contextualized features to multiple prototypes per class, learning nonuniform temporal assignments while discouraging excessive concentration. Candidate-restricted inference and dual-marginal KL projection construct a teacher that integrates momentum-refined class beliefs while preserving the proposal's temporal occupancy. A single plan-level KL objective aligns the full-space predictor with this teacher. Our analysis characterizes when candidate re-solving differs from masking and establishes, under the proposed construction, the objective's decomposition into class-marginal alignment and class-conditional temporal supervision. Across Breakfast-MIPL, DoTA-MIPL, and FineAction-MIPL with four feature representations, {\ours} achieves the highest mean ACC and BACC in 10 of 12 configurations. Ablations support the joint objective and dual-marginal teacher, while encoder-matched comparisons suggest gains beyond temporal encoding. These findings support resolving label ambiguity through explicit class--time associations rather than class probabilities alone.

\bibliography{2027_conference}
\bibliographystyle{2027_conference}

\clearpage
\appendix
\section{Structural Properties and Proofs}
\label{app:revth:all}

We analyze the optimization problems defined in the paper, rather than the
convergence of network training. Throughout this appendix, one bag has $n\geq1$
valid temporal positions, finite costs $C_{ckt}$, $K\geq1$ prototypes per class,
and a nonempty admissible class set $\mathcal A$. We fix $\varepsilon>0$,
$\lambda\geq0$, and $u_t=1/n$. Padding positions are excluded from this
domain. All plans below are exact optima unless numerical approximations are
explicitly discussed. Derivative statements additionally assume that the actor is differentiable at the parameter point under consideration. These properties do not imply recovery of the true label,
correct temporal localization, or generalization of the learned network.

\subsection{Factorization and Uniqueness}
\label{app:revth:occupancy}

Write $\Pi_{ckt}=m_t\rho_{ckt}$, where $\sum_t m_t=1$ and
$\sum_{c\in\mathcal A,k}\rho_{ckt}=1$. For any positive $m_t$, conditional
minimization over $\rho_{\cdot\cdot t}$ gives
\begin{equation}
 Z_t^{\mathcal A}=\sum_{c\in\mathcal A,k}e^{-C_{ckt}/\varepsilon},\qquad
 \rho_{ckt}^{\mathcal A}=\frac{e^{-C_{ckt}/\varepsilon}}{Z_t^{\mathcal A}},
 \qquad f_t^{\mathcal A}=-\varepsilon\log Z_t^{\mathcal A}.
 \label{eq:revth:conditional}
\end{equation}
Indeed, its cost plus conditional entropy equals
$f_t^{\mathcal A}+\varepsilon D_{\mathrm{KL}}
(\rho_{\cdot\cdot t}\|\rho^{\mathcal A}_{\cdot\cdot t})$.
The full objective therefore reduces, up to a constant independent of $m$, to
\begin{equation}
 G_\lambda(m)=\langle f,m\rangle+
 \varepsilon\sum_t m_t\log(m_t/u_t)
 +\frac{\lambda}{2}R(m),\qquad
 R(m)=n\sum_t m_t^2-1,\qquad m\in\Delta_n.
 \label{eq:revth:reduced}
\end{equation}
This is an optimization over an $n$-dimensional vector subject to one simplex
constraint; only its normalization multiplier is a scalar root variable.

\begin{proposition}[Unique Interior Optimum]
\label{prop:unique-interior-optimum}
For finite $f$, the objective in Eq.~(\ref{eq:revth:reduced}) has a unique minimizer
$m^\star$ with $m_t^\star>0$ for every valid $t$. Its necessary and sufficient
optimality conditions are
\begin{equation}
 f_t+\varepsilon\log(m_t^\star/u_t)+\lambda n m_t^\star=\xi,
 \qquad \sum_t m_t^\star=1.
 \label{eq:revth:kkt}
\end{equation}
\end{proposition}
\begin{proof}
The continuous extension of $x\log(x/u)$ at zero is strictly convex on
$[0,\infty)$. For positive arguments, this follows from its second derivative
$1/x$; if one argument is zero, for $y>0$ and $a\in(0,1)$,
$(ay)\log(ay/u)=a y\log(y/u)+ay\log a<a y\log(y/u)$.
Thus, the entropy is strictly convex on the whole simplex. The other terms are
linear or convex, and compactness gives existence and uniqueness. A boundary
point with $m_t=0$ can be perturbed by transferring $\delta>0$ mass from a
positive component to $t$. Its objective difference is
$\varepsilon\delta\log\delta+O(\delta)<0$ for small $\delta$, so it cannot
minimize the objective. At the interior optimum, differentiating the
Lagrangian and absorbing its common entropy constant into $\xi$ gives
Eq.~(\ref{eq:revth:kkt}). Strict convexity makes these conditions sufficient.
For $n=1$, the unique feasible point is already positive.
\end{proof}
For $\lambda>0$, the principal Lambert-$W$ function yields
\begin{equation}
 m_t(\xi)=\frac{\varepsilon}{\lambda n}
 W\!\left(\frac{\lambda}{\varepsilon}
       e^{(\xi-f_t)/\varepsilon}\right),\qquad
 H(\xi)=\sum_t m_t(\xi)-1.
 \label{eq:revth:lambert}
\end{equation}
Implicit differentiation of Eq.~(\ref{eq:revth:kkt}) at fixed $f$ gives
\begin{equation}
 H'(\xi)=\sum_t\frac{m_t(\xi)}{\varepsilon+\lambda n m_t(\xi)}>0.
 \label{eq:revth:rootderivative}
\end{equation}
The sum in $H$ tends to zero as $\xi\to-\infty$ and to infinity as
$\xi\to\infty$, so the root is unique. It obeys the finite bracket
\begin{equation}
 \min_t f_t+\lambda\ \leq\ \xi^\star\ \leq\ \max_t f_t+\lambda.
 \label{eq:revth:bracket}
\end{equation}
To see this, the left side of Eq.~(\ref{eq:revth:kkt}), as a function of $m_t$,
is strictly increasing and equals $f_t+\lambda$ at $m_t=u_t$.
At the lower bracket, each $m_t(\xi)\leq u_t$, and at the upper bracket, each
$m_t(\xi)\geq u_t$. Eq.~(\ref{eq:revth:bracket}) also provides a valid bracket for a safeguarded Newton--bisection solver. Our implementation instead uses the fixed-iteration Newton procedure described in Appendix~\ref{app:implementation}, without bisection. If all $f_t$ are equal, the bracket is a single exact root.

The case $\lambda=0$ must be handled separately:
\begin{equation}
 m_t^\star=\frac{e^{-f_t/\varepsilon}}{\sum_s e^{-f_s/\varepsilon}},\qquad
 \xi^\star=-\varepsilon\log\sum_t u_t e^{-f_t/\varepsilon}.
 \label{eq:revth:gibbsoccupancy}
\end{equation}
It retains entropic smoothing. Substitution of zero into the expression with
$1/\lambda$ in Eq.~(\ref{eq:revth:lambert}) is invalid.

\subsection{Candidate-Restricted Re-solving and Global Masking}
\label{app:revth:mask}

Let $\Pi^A$ solve the problem over all classes, and let
$\widetilde\Pi$ solve it over a nonempty candidate set $S$, using exactly the
same $C,\varepsilon,\lambda,n,u$. Define the globally masked actor plan
and its retained conditional mass by
\begin{equation}
 a_t=\sum_{c\in S,k}\rho^A_{ckt}>0,\quad
 Q=\sum_t m_t^A a_t>0,\quad
 \Pi^R_{ckt}=\frac{\mathbf1[c\in S]\Pi^A_{ckt}}{Q},\quad
 m_t^R=\frac{m_t^A a_t}{Q}.
 \label{eq:revth:mask}
\end{equation}
All comparisons with $\widetilde\Pi$ extend its noncandidate entries by zero.
Thus global masking generally changes temporal occupancy. This operation
differs both from masking only a class probability vector and from
renormalizing the conditional class distribution separately at each time
while explicitly retaining $m^A$.

\begin{proposition}[When Re-solving Differs From Masking]\label{prop:mask_main}
If $\lambda=0$, then $\widetilde\Pi=\Pi^R$. If $\lambda>0$, then
$\widetilde\Pi=\Pi^R$ if and only if $a_t$ is constant over valid temporal
positions.
\end{proposition}

\begin{proof}
The restricted conditional distribution is
$\widetilde\rho_{ckt}=\rho^A_{ckt}/a_t$ for $c\in S$, and
$\widetilde f_t=f_t^A-\varepsilon\log a_t$.
The conditional distribution of $\Pi^R$ is exactly
$\widetilde\rho$, so only the occupancies can differ. When $\lambda=0$,
Eq. (\ref{eq:revth:gibbsoccupancy}) and Eq. (\ref{eq:revth:conditional}) imply
$\Pi^A_{ckt}\propto e^{-C_{ckt}/\varepsilon}$ globally; restricting and
normalizing this Gibbs distribution is the restricted optimum.
For $\lambda>0$, insertion of $m^R$ into the restricted stationarity
expression gives
\begin{equation}
 \widetilde f_t+\varepsilon\log(m_t^R/u_t)+\lambda n m_t^R
 =\xi^A-\varepsilon\log Q+\lambda n(m_t^R-m_t^A).
 \label{eq:revth:maskresidual}
\end{equation}
If $\Pi^R$ is optimal, the right side must be constant in $t$.
Since $\lambda n>0$, $m_t^R-m_t^A$ must be constant; its sum is zero, so
$m^R=m^A$. Positivity of $m^A$ and Eq.~(\ref{eq:revth:mask}) imply $a_t=Q$
for all $t$. Conversely, constant $a_t$ gives $m^R=m^A$, making
Eq.~(\ref{eq:revth:maskresidual}) constant. The unique interior optimum
established above then implies $\Pi^R=\widetilde\Pi$.
\end{proof}
Consequently, re-solving enforces the occupancy-regularized optimality
conditions within the candidate domain. Its distinction from global masking
does not follow merely from the fact that temporal occupancy changes.
The result assumes the uniform prior and exact objectives stated above; it
does not assert equivalence for modified priors, regularizers, or solvers.

\paragraph{Why Retain The Association Between Class And Time?}
As an illustration, consider two generic class--time distributions,
\begin{equation}
 U=\begin{pmatrix}0.4&0.1\\0.1&0.4\end{pmatrix},\qquad
 V=\begin{pmatrix}0.1&0.4\\0.4&0.1\end{pmatrix}.
 \label{eq:marginal_example}
\end{equation}
Both have class and temporal marginals $(0.5,0.5)$, yet associate each class with a different time. Their class-marginal KL is zero, while $D_{\rm KL}(U\|V)=0.6\log4>0$. This example concerns distributions, not measured video predictions or an asserted actor--teacher pair. It shows why marginal agreement alone leaves the class--time association unspecified. In our construction, proximity to the proposal selects a joint target among the plans that satisfy both marginals.

\subsection{Teacher projection and the effective supervision space}
\label{app:revth:kl}

Let $q$ be the refined candidate-class marginal and $\widetilde m$ the
current proposal occupancy. Finite costs and $\varepsilon>0$ give a positive
proposal on $S\times\{1,\ldots,K\}\times\{1,\ldots,n\}$. The stated
initialization and EMA with $\mu\in[0,1)$ give $q_c>0$ on $S$.
The KL projection with class marginal $q$ and temporal marginal
$\widetilde m$ has a unique solution. A strictly positive feasible point
exists, for example $q_c\widetilde m_t/K$; the entropy boundary argument
ensures that the minimizing plan is positive on this domain. Its Lagrange
stationarity condition gives
\begin{equation}
 \Pi^T_{ckt}=\alpha_c\beta_t\widetilde\Pi_{ckt},\qquad
 \alpha_c,\beta_t>0.
 \label{eq:revth:scaling}
\end{equation}
Although the factors have a reciprocal scaling ambiguity, the plan is unique.
With $\widetilde M_{ct}=\sum_k\widetilde\Pi_{ckt}$, alternating marginal
scaling is
\begin{equation}
 \alpha_c\leftarrow\frac{q_c}{\sum_t\widetilde M_{ct}\beta_t},\qquad
 \beta_t\leftarrow\frac{\widetilde m_t}{\sum_{c\in S}\widetilde M_{ct}\alpha_c}.
 \label{eq:revth:ipfp}
\end{equation}
These updates need only the $|S|\times n$ class--time matrix. They can also be
implemented in the log domain. If a target class has exactly zero mass in an
alternative initialization, remove that row and apply the argument on the
positive target support; extend it by zeros afterwards.

Assume that actor and proposal are constructed at the same parameter point
$\theta_0$, using the same costs, prototypes, features, and $\varepsilon$,
and that the teacher correction depends only on $(c,t)$. Then
\begin{equation}
 r^A_{k\mid c,t}=r^T_{k\mid c,t}
 =\frac{e^{-C_{ckt}/\varepsilon}}{\sum_{k'}e^{-C_{ck't}/\varepsilon}}
 \quad(c\in S).
 \label{eq:revth:prototypeconditional}
\end{equation}
The temporal multiplier in the actor, the restricted partition function in
the proposal, and the teacher scaling in Eq.~(\ref{eq:revth:scaling}) all cancel
when conditioning on $(c,t)$.

\begin{proposition}[KL and Current-point Gradient Equivalence]\label{prop:kl_main}
Set $M^{A/T}_{ct}=\sum_k\Pi^{A/T}_{ckt}$ and extend teacher plans and
marginals with zeros outside $S$ (denoted by hats). Under the conditions just
stated,
\begin{align}
 D_{\mathrm{KL}}(\widehat\Pi^T\|\Pi^A)
 &=D_{\mathrm{KL}}(\widehat M^T\|M^A),
 \label{eq:revth:kl2d}\\
 &=D_{\mathrm{KL}}(q^T\|q^A)
 +\sum_{c:q_c^T>0}q_c^T
 D_{\mathrm{KL}}\bigl(P_T(t\mid c)\|P_A(t\mid c)\bigr),
 \label{eq:revth:klclass}
\end{align}
where $q_c^T=\sum_t\widehat M^T_{ct}$ is the actual teacher class marginal.
If this teacher is frozen, the three-dimensional and two-dimensional losses
also have equal first derivatives with respect to actor parameters at
$\theta=\theta_0$.
\end{proposition}

\begin{proof}
The KL chain rule applied to $(c,t)$ and then to $k$ gives
\begin{equation}
 D_{\mathrm{KL}}(\widehat\Pi^T\|\Pi^A)
 =D_{\mathrm{KL}}(\widehat M^T\|M^A)
 +\sum_{c,t:\widehat M^T_{ct}>0}\widehat M^T_{ct}
 D_{\mathrm{KL}}(r^T_{\cdot\mid c,t}\|r^A_{\cdot\mid c,t}).
 \label{eq:revth:klchain}
\end{equation}
Eq.~(\ref{eq:revth:prototypeconditional}) makes the last term zero.
Applying the same chain rule to class and time in $M$ proves
Eq.~(\ref{eq:revth:klclass}).
For the derivative statement, denote the teacher built at $\theta_0$ by
$T_0$ and hold it fixed while differentiating the actor. The difference
between the two loss gradients is
\begin{equation}
 -\sum_{c,t}\widehat M^{T_0}_{ct}\sum_k
 r^{T_0}_{k\mid c,t}\nabla_\theta\log r^A_{k\mid c,t}(\theta).
 \label{eq:revth:gradientdifference}
\end{equation}
At $\theta_0$, $r^{T_0}=r^A(\theta_0)$, so every inner sum is
$\sum_k\nabla_\theta r^A_{k\mid c,t}(\theta_0)
=\nabla_\theta1=0$. This proves local first-order equivalence despite
stopping gradients through the teacher; equality of values alone would
not justify that derivative conclusion.
\end{proof}
The extra supervision beyond the class marginal is therefore the
class-conditional temporal distribution, parameterized by multi-prototype
geometry. Prototype parameters still influence
$\sum_k e^{-C_{ckt}/\varepsilon}$ and generally receive nonzero gradients.
There is no additional nonzero prototype-conditional KL term under these
conditions. The result does not extend to stale teachers reused after actor
updates, different teacher costs or temperatures, prototype-dependent
corrections, or higher derivatives of a frozen-teacher loss. For an
approximate multiplicative projection, conditional preservation still holds,
but Eq.~(\ref{eq:revth:klclass}) uses its actual marginal $q^T$, which need not
equal the requested $q$ until projection residuals vanish.

\subsection{Teacher Projection and Class–Time Supervision}
\label{app:revth:degeneracy}

\begin{proposition}[No Momentum And No Occupancy Penalty]
Under the same-cost and exact-solution assumptions, let $\mu=0$ and
$\lambda=0$. Then the teacher equals the candidate proposal and
\begin{equation}
 D_{\mathrm{KL}}(\widehat\Pi^T\|\Pi^A)
 =-\log Q,\qquad Q=\sum_{c\in S}q_c^A.
 \label{eq:revth:candidatemass}
\end{equation}
The stopped-teacher gradient at its construction point also equals the
gradient of $-\log Q$.
\end{proposition}
\begin{proof}
When $\mu=0$, the refined marginal is the proposal class marginal.
The proposal already satisfies both teacher constraints and has zero KL
distance to itself, so the uniqueness of the projection implies
$\Pi^T=\widetilde\Pi$. The masking proposition with $\lambda=0$ gives
$\widehat\Pi^T=\mathbf1[c\in S]\Pi^A/Q$. The log ratio on its support
is the constant $-\log Q$, proving the value identity. At the teacher
construction point, differentiation with that teacher fixed gives
\begin{equation}
 -\sum_{c\in S,k,t}\frac{\Pi^A_{ckt}}{Q}\nabla_\theta\log\Pi^A_{ckt}
 =-\frac1Q\nabla_\theta\sum_{c\in S,k,t}\Pi^A_{ckt}
 =\nabla_\theta[-\log Q].
 \label{eq:revth:candidatemassgradient}
\end{equation}
No derivative through the teacher is used.
\end{proof}
This is a boundary configuration, not an identity for the full method with
positive momentum or occupancy regularization. It motivates a controlled
candidate-mass baseline without asserting its relative empirical performance.

\subsection{Effective-Support Control by Occupancy Regularization}
\label{app:revth:support}

Define $N_{\mathrm{eff}}(m)=1/\sum_t m_t^2$, so that
$1\leq N_{\mathrm{eff}}(m)\leq n$ and
\begin{equation}
 R(m)=\frac{n}{N_{\mathrm{eff}}(m)}-1
 =\chi^2(m\|u)=\exp(D_2(m\|u))-1.
 \label{eq:revth:effectivesupport}
\end{equation}
\begin{proposition}[Fixed-geometry Monotonicity]
For fixed $f,\varepsilon>0,n$, $R(m_\lambda^\star)$ is nonincreasing and
$N_{\mathrm{eff}}(m_\lambda^\star)$ is nondecreasing in $\lambda\geq0$.
Moreover, $m_\lambda^\star\to u$ as $\lambda\to\infty$.
\end{proposition}
\begin{proof}
Write $F(m)=\langle f,m\rangle+\varepsilon D_{\mathrm{KL}}(m\|u)$.
For $\lambda_2>\lambda_1\geq0$, optimality at the two coefficients gives
\begin{align*}
 F(m_1)+\tfrac{\lambda_1}{2}R(m_1)
 &\leq F(m_2)+\tfrac{\lambda_1}{2}R(m_2),\\
 F(m_2)+\tfrac{\lambda_2}{2}R(m_2)
 &\leq F(m_1)+\tfrac{\lambda_2}{2}R(m_1).
\end{align*}
Adding them yields
$(\lambda_2-\lambda_1)(R(m_2)-R(m_1))\leq0$.
Eq.~(\ref{eq:revth:effectivesupport}) gives the support statement.
Since $R(u)=0$, comparison with the feasible $u$ gives, for $\lambda>0$,
\begin{equation}
 0\leq R(m_\lambda^\star)
 \leq\frac{2[F(u)-\min_{m\in\Delta_n}F(m)]}{\lambda}.
 \label{eq:revth:uniformlimitbound}
\end{equation}
The numerator is finite by compactness. As
$R(m)=n\|m-u\|_2^2$, the bound implies convergence to $u$.
At $\lambda=0$, Eq.~(\ref{eq:revth:gibbsoccupancy}) gives the entropy-smoothed
Gibbs occupancy.
\end{proof}
This statement holds with geometry fixed. During learning $f$ changes, so it
does not establish monotonic support across training updates or monotonic
accuracy in $\lambda$. Effective support measures mass concentration, not
semantic diversity or true-action coverage; overly diffuse mass can also
assign weight to the background around a short action.

\subsection{Differentiability and Computational Complexity}
\label{app:revth:numerics}

For an exact interior solution, set
$w_t=(\varepsilon/m_t^\star+\lambda n)^{-1}>0$.
Differentiating Eq.~(\ref{eq:revth:kkt}) gives
$dm_t=w_t(d\xi-df_t)$, and $\sum_t dm_t=0$ implies
$d\xi=(\sum_t w_tdf_t)/(\sum_t w_t)$. Hence
\begin{equation}
 \frac{\partial m^\star}{\partial f}
 =-\operatorname{diag}(w)+\frac{ww^\top}{\sum_t w_t}.
 \label{eq:revth:jacobian}
\end{equation}
This matrix is symmetric and negative semidefinite: for any vector $v$, its
quadratic form is the negative weighted variance
$-\sum_t w_tv_t^2+(\sum_t w_tv_t)^2/\sum_t w_t\leq0$.
Its null direction $\mathbf1$ reflects invariance to adding a common
constant to $f$. The formula also covers $\lambda=0$, yielding the usual
softmax Jacobian with a negative inverse-temperature factor. A
Jacobian--vector product costs $O(n)$ without materializing an $n\times n$
matrix. Differentiation of the complete actor must additionally include
$\partial f_t/\partial C_{ckt}=\rho_{ckt}$ and the conditional Gibbs
derivative; retaining only the occupancy path would change the gradient.

For feature dimension $d$, the dense spherical matching costs
$O(CKnd)$ operations and $O(CKn)$ storage when scores are materialized.
Given these scores, the conditional distributions and free energies cost
$O(CKn)$. With $J_A$ and $J_S$ scalar root iterations, the actor and
proposal occupancy solves cost $O((J_A+J_S)n)$, in addition to the
$O(|S|Kn)$ restricted reductions. With $I$ marginal-scaling iterations,
projection on $\widetilde M$ costs $O(I|S|n)$; a three-dimensional teacher,
if materialized, costs $O(|S|Kn)$ storage. These counts exclude the temporal
encoder and depend on the requested numerical precision. The scalar
normalization root does not remove the class--prototype matching cost. \looseness=-1

For numerical evaluation, define $h_t=\log(\lambda/\varepsilon)+(\xi-f_t)/\varepsilon$ and let $\omega_t=W(e^{h_t})$. Rather than explicitly forming $e^{h_t}$, we evaluate $\omega_t$ from $\omega_t+\log\omega_t=h_t$ using stable Newton updates and recover $m_t=\varepsilon\omega_t/(\lambda n)$. The implementation uses eight fixed Lambert-$W$ updates per evaluation and twelve Newton iterations for the normalization root, with the latter performed in double precision. The $\lambda=0$ branch uses the closed-form softmax solution.

An approximate solver should report its mass residual
$|\sum_t m_t-1|$, stationarity residual
$\max_t|f_t+\varepsilon\log(m_t/u_t)+\lambda n m_t-\xi|$, and both teacher
marginal residuals
$\|\sum_{k,t}\Pi^T_{ckt}-q_c\|_\infty$ and
$\|\sum_{c,k}\Pi^T_{ckt}-\widetilde m_t\|_\infty$.

\section{Dataset construction and implementation details}\label{app:protocol}

\subsection{More Information on Benchmark Construction}

We construct VCMIPL benchmarks from Breakfast~\citep{kuehne2014breakfast}, DoTA~\citep{yao2023dota}, and FineAction~\citep{liu2022fineaction}. Table~\ref{tab:dataset_statistics} summarizes the video counts, class counts, and candidate-set sizes. Each example consists of an ordered video bag and a candidate label set containing its ground-truth class. For FineAction, we retain videos with a single annotated target class and exclude classes containing fewer than 50 videos. The resulting benchmarks contain 1,989 videos from 10 classes for Breakfast, 4,577 videos from 9 classes for DoTA, and 10,675 videos from 78 classes for FineAction. 
In the output space, we generate candidate labels using a VideoMAEv2 recognizer equipped with a TimeSformerHead in MMAction2~\citep{2020mmaction2}. Candidate probabilities are obtained through five-fold cross-fitting: the recognizer is fine-tuned using the exact video labels from four folds and predicts the remaining fold. Repeating this procedure produces an out-of-fold prediction for every video. The recognizer is fine-tuned for 30 epochs on DoTA and 100 epochs on both Breakfast and FineAction.  \looseness=-1

Following the sample-wise confidence accumulation and class-wise selection principles of CPL~\citep{zhang2024candidate}, we combine two criteria. Let $\mathbf p_i=(p_{i1},\ldots,p_{iC})$ denote the out-of-fold class-probability vector for video $i$, and let $c_{i1},\ldots,c_{iC}$ order the classes by decreasing confidence. The sample-wise candidate set contains the smallest number of classes whose cumulative probability reaches $95\%$:
\begin{equation}
A_i=\{c_{i1},\ldots,c_{i\ell_i}\},\qquad
\ell_i=\min\left\{\ell\in\{1,\ldots,C\}:\sum_{j=1}^{\ell}p_{i,c_{ij}}\geq0.95\right\}.
\label{eq:app:samplewise_candidates}
\end{equation}
For each class $c$, let $B_c$ denote the indices of the top $50\%$ of videos in the class-wise ranking population, ordered by their confidence for that class. Intersecting the two selections gives
\begin{equation}
\widetilde S_i=A_i\cap\{c\in\mathcal Y:i\in B_c\}.
\label{eq:app:intersected_candidates}
\end{equation}
To satisfy the standard MIPL assumption, we include the ground-truth class whenever it is absent:
\begin{equation}
S_i=\widetilde S_i\cup\{Y_i\}.
\label{eq:app:final_candidates}
\end{equation}
Consequently, the final candidate set is nonempty and satisfies $Y_i\in S_i$ by construction. This guarantee applies to the corrected candidate sets, rather than to the recognizer's predictions before correction. Exact labels are used to train the candidate recognizer and perform this inclusion step during benchmark construction; downstream MIPL learning uses only the resulting candidate sets as supervision.

\begin{table}[!t]
\centering
\caption{Video-level statistics of the constructed MIPL benchmarks.}
\label{tab:dataset_statistics}
\begin{tabular}{lrrcc}
\toprule
Dataset & Videos & Classes & Mean $|S_i|$ & Range $|S_i|$\\
\midrule
Breakfast & 1,989 & 10 & 3.3896 & 1--9\\
DoTA & 4,577 & 9 & 3.2733 & 1--8\\
FineAction & 10,675 & 78 & 11.5592 & 2--64\\
\bottomrule
\end{tabular}
\end{table}

\subsection{Feature Representations and Dataset Statistics}
ResNet~\citep{he2016resnet} and DINOv3~\citep{simeoni2025dinov3} provide frame-level features, whereas SlowFast~\citep{feichtenhofer2019slowfast} and VideoMAEv2~\citep{wang2023videomaev2} provide clip-level features. These representations yield twelve dataset--feature configurations spanning different temporal granularities and feature dimensions.

Table~\ref{tab:dataset_statistics_full} reports the complete instance statistics. Total instances denotes the number of extracted frame or clip features across all videos; maximum, minimum, and mean refer to the number of instances per video. Dimension denotes the width of each feature vector. Within each dataset, the four representations share the video, class, and candidate-set statistics in Table~\ref{tab:dataset_statistics}. The fine-tuning budgets in the dataset headers refer to the candidate recognizer, not to the individual feature extractors.  \looseness=-1

\begin{table}[htbp]
\centering\small
\caption{Instance and feature statistics for all twelve dataset--feature configurations.}
\label{tab:dataset_statistics_full}
\setlength{\tabcolsep}{5pt}
\begin{tabular}{lrrrrr}
\toprule
Feature & Total instances & Maximum & Minimum & Mean & Dimension\\
\midrule
\multicolumn{6}{c}{\textbf{Breakfast: 1,989 videos, 10 classes, 100 fine-tuning epochs}}\\
\multicolumn{6}{c}{Candidate size: 1--9, mean 3.3896}\\
\midrule
ResNet & 4,103,316 & 9,745 & 186 & 2,063.00 & 2,048\\
DINOv3 & 4,103,316 & 9,745 & 186 & 2,063.00 & 4,096\\
SlowFast & 1,026,577 & 2,437 & 47 & 516.13 & 2,304\\
VideoMAEv2 & 1,026,577 & 2,437 & 47 & 516.13 & 768\\
\midrule
\multicolumn{6}{c}{\textbf{DoTA: 4,577 videos, 9 classes, 30 fine-tuning epochs}}\\
\multicolumn{6}{c}{Candidate size: 1--8, mean 3.2733}\\
\midrule
ResNet & 469,756 & 300 & 26 & 102.63 & 2,048\\
DINOv3 & 469,756 & 300 & 26 & 102.63 & 4,096\\
SlowFast & 119,188 & 75 & 7 & 26.04 & 2,304\\
VideoMAEv2 & 119,188 & 75 & 7 & 26.04 & 768\\
\midrule
\multicolumn{6}{c}{\textbf{FineAction: 10,675 videos, 78 classes, 100 fine-tuning epochs}}\\
\multicolumn{6}{c}{Candidate size: 2--64, mean 11.5592}\\
\midrule
ResNet & 37,723,340 & 64,644 & 36 & 3,533.80 & 2,048\\
DINOv3 & 37,723,340 & 64,644 & 36 & 3,533.80 & 4,096\\
SlowFast & 9,518,494 & 47,372 & 9 & 891.66 & 2,304\\
VideoMAEv2 & 9,518,494 & 47,372 & 9 & 891.66 & 768\\
\bottomrule
\end{tabular}
\end{table}

\subsection{More Implementation Settings}\label{app:implementation}
We adopt a fixed five-fold evaluation protocol for all three benchmarks. In each fold, four predefined folds are used for training and the remaining fold for testing, and the reported ACC and BACC are averaged over the five held-out folds. Candidate labels are generated in an out-of-fold manner using the same fold partition: the candidate recognizer is fine-tuned on four folds and applied to the held-out fold, and this procedure is repeated until every video receives an out-of-fold prediction. The downstream MIPL models are then trained using only the resulting candidate label sets as supervision. {\ours} is optimized for 40 epochs with one complete video bag per optimization step, without temporal padding, and the same training protocol is used across all feature representations unless otherwise specified. Ablation studies are conducted on Fold~1 and evaluated at epoch~40. Detailed optimization, architecture, transport, and teacher-projection settings are summarized in Table~\ref{tab:implementation_settings}. \looseness=-1

\begin{table}[!t]
\centering
\small
\setlength{\tabcolsep}{4pt}
\renewcommand{\arraystretch}{1.10}
\caption{Experimental settings and configurable implementation defaults of {\ours}.}
\label{tab:implementation_settings}
\begin{tabular}{@{}p{0.34\linewidth}p{0.62\linewidth}@{}}
\toprule
\textbf{Setting} & \textbf{Value or implementation}\\
Framework / hardware & PyTorch / NVIDIA H100\\
Optimizer & AdamW; $(\beta_1,\beta_2)=(0.9,0.999)$; weight decay $0$\\
Learning rates & Prototypes: $5\times10^{-3}$; temporal encoder: $5\times10^{-4}$\\
Gradient clipping & Global $\ell_2$ norm clipped to $5.0$\\
Training budget & $40$ epochs\\
Motion module & Average-filter kernel $3$; convolution kernel $5$; dilations $1,2$; motion-residual dropout $0.1$\\
Candidate recognizer budget & DoTA: $30$ epochs; Breakfast and FineAction: $100$ epochs\\
Main evaluation & Five folds; mean and standard deviation of ACC and BACC\\
Ablation evaluation & Fold~1 ACC at epoch~40\\
Learning-rate schedule & Constant; no warm-up or scheduler\\
Batch size & One complete video bag per optimization step\\
Prototypes per class & Breakfast: $K=4$; DoTA: $K=8$; FineAction: $K=12$\\
Class momentum & $\mu=0.9$\\
Transport regularization & Entropic coefficient $\varepsilon=0.07$; occupancy coefficient $\lambda=0.01$\\
Temporal encoder & Hidden dimension $512$; $4$ bidirectional Mamba-2 blocks\\
Feed-forward network & Expansion factor $4$; intermediate dimension $2048$\\
Mamba-2 configuration & State dimension $128$; convolution width $4$; expansion factor $2$\\
Mamba-2 computation & Head dimension $64$; chunk size $256$; memory-efficient path enabled\\
Teacher projection & At most $120$ IPFP iterations; marginal tolerance $10^{-6}$\\
Base random seed & $123$, with fold-dependent offsets\\
\bottomrule
\end{tabular}
\end{table}

\begingroup
\raggedbottom
\setlength{\textfloatsep}{12pt plus 2pt minus 2pt}
\setlength{\floatsep}{12pt plus 2pt minus 2pt}
\setlength{\intextsep}{8pt plus 2pt minus 2pt}
\renewcommand{\bottomfraction}{0.95}
\renewcommand{\textfraction}{0.05}
\renewcommand{\floatpagefraction}{0.85}
\makeatletter
\setlength{\@fptop}{0pt}
\setlength{\@fpsep}{12pt}
\makeatother
\section{Additional Experimental Results}\label{app:results}
\subsection{Balanced Accuracy Results}\label{BACC}
Table~\ref{tab:overall-BACC-comparison} further confirms the advantage of {\ours} under class-balanced evaluation. {\ours} achieves the highest mean BACC, the mean recall over ground-truth classes represented in each test fold, in 10 of the 12 dataset--feature configurations, with an average gain of 3.26\% over the strongest baseline. Notably, the improvement pattern closely follows that of ACC, indicating that the gains are not primarily driven by better prediction on dominant classes, but extend across the class distribution. {\ours} also ranks first in all four DoTA-MIPL configurations, where the substantially lower BACC values make balanced recognition particularly challenging. The only exceptions again occur with VideoMAEv2, suggesting that the reduced headroom observed in the ACC results persists under class-balanced evaluation.
\begin{table}[!htbp]
\centering
\small
\renewcommand{\arraystretch}{0.95}
\setlength{\tabcolsep}{8pt}
\caption{Balanced accuracy on the benchmark datasets. Results are reported as (mean $\pm$ std \%). The best result in each column is bolded, and the second-best results are underlined.}\label{tab:overall-BACC-comparison}
\begin{tabular}{llcccc}
\toprule
Dataset & Method & DINOv3 & VideoMAEv2 & ResNet & SlowFast\\
\midrule
\multirow{8}{*}{\textbf{\shortstack{Breakfast\\-MIPL}}}
& \demipl & $10.2\%\,{\scriptstyle\pm\,0.4\%}$ & $24.4\%\,{\scriptstyle\pm\,10.2\%}$ & $10.6\%\,{\scriptstyle\pm\,0.6\%}$ & $10.0\%\,{\scriptstyle\pm\,0.0\%}$\\*
& \elimipl & $27.9\%\,{\scriptstyle\pm\,0.8\%}$ & $52.4\%\,{\scriptstyle\pm\,2.1\%}$ & $26.0\%\,{\scriptstyle\pm\,1.9\%}$ & $10.0\%\,{\scriptstyle\pm\,0.0\%}$\\*
& \miplma & $34.2\%\,{\scriptstyle\pm\,1.5\%}$ & $63.8\%\,{\scriptstyle\pm\,2.0\%}$ & $30.8\%\,{\scriptstyle\pm\,2.6\%}$ & $10.0\%\,{\scriptstyle\pm\,0.0\%}$\\*
& \promipl & $\underline{61.9\%}\,{\scriptstyle\pm\,1.1\%}$ & $83.3\%\,{\scriptstyle\pm\,3.5\%}$ & $\underline{43.7\%}\,{\scriptstyle\pm\,1.0\%}$ & $41.4\%\,{\scriptstyle\pm\,1.1\%}$\\*
& \psmipl & $47.7\%\,{\scriptstyle\pm\,4.9\%}$ & $70.1\%\,{\scriptstyle\pm\,7.1\%}$ & $10.9\%\,{\scriptstyle\pm\,1.4\%}$ & $32.9\%\,{\scriptstyle\pm\,0.5\%}$\\*
& \dualg & $45.2\%\,{\scriptstyle\pm\,2.7\%}$ & $65.1\%\,{\scriptstyle\pm\,4.1\%}$ & $34.7\%\,{\scriptstyle\pm\,2.2\%}$ & $13.8\%\,{\scriptstyle\pm\,1.7\%}$\\*
& \agopmipl & $46.0\%\,{\scriptstyle\pm\,4.3\%}$ & $\mathbf{87.2\%}\,{\scriptstyle\pm\,4.3\%}$ & $33.3\%\,{\scriptstyle\pm\,3.3\%}$ & $\underline{50.1\%}\,{\scriptstyle\pm\,2.4\%}$\\*
& {\ours} & $\mathbf{66.0\%}\,{\scriptstyle\pm\,1.0\%}$ & $\underline{85.2\%}\,{\scriptstyle\pm\,4.3\%}$ & $\mathbf{52.2\%}\,{\scriptstyle\pm\,3.3\%}$ & $\mathbf{55.8\%}\,{\scriptstyle\pm\,1.8\%}$\\
\midrule
\multirow{8}{*}{\textbf{\shortstack{DoTA\\-MIPL}}}
& \demipl & $24.8\%\,{\scriptstyle\pm\,0.7\%}$ & $42.7\%\,{\scriptstyle\pm\,0.5\%}$ & $23.2\%\,{\scriptstyle\pm\,0.4\%}$ & $18.2\%\,{\scriptstyle\pm\,2.6\%}$\\*
& \elimipl & $24.8\%\,{\scriptstyle\pm\,0.8\%}$ & $42.9\%\,{\scriptstyle\pm\,0.8\%}$ & $23.9\%\,{\scriptstyle\pm\,0.6\%}$ & $21.2\%\,{\scriptstyle\pm\,1.7\%}$\\*
& \miplma & $25.8\%\,{\scriptstyle\pm\,0.6\%}$ & $43.2\%\,{\scriptstyle\pm\,0.5\%}$ & $23.9\%\,{\scriptstyle\pm\,0.5\%}$ & $21.5\%\,{\scriptstyle\pm\,3.4\%}$\\*
& \promipl & $\underline{32.7\%}\,{\scriptstyle\pm\,1.6\%}$ & $58.0\%\,{\scriptstyle\pm\,1.5\%}$ & $\underline{28.8\%}\,{\scriptstyle\pm\,2.0\%}$ & $\underline{28.8\%}\,{\scriptstyle\pm\,0.8\%}$\\*
& \psmipl & $30.0\%\,{\scriptstyle\pm\,0.9\%}$ & $\underline{58.3\%}\,{\scriptstyle\pm\,1.4\%}$ & $26.2\%\,{\scriptstyle\pm\,0.5\%}$ & $14.9\%\,{\scriptstyle\pm\,4.7\%}$\\*
& \dualg & $26.0\%\,{\scriptstyle\pm\,0.7\%}$ & $44.1\%\,{\scriptstyle\pm\,0.6\%}$ & $24.9\%\,{\scriptstyle\pm\,0.8\%}$ & $24.2\%\,{\scriptstyle\pm\,0.9\%}$\\*
& \agopmipl & $28.6\%\,{\scriptstyle\pm\,0.9\%}$ & $51.4\%\,{\scriptstyle\pm\,3.1\%}$ & $26.1\%\,{\scriptstyle\pm\,1.0\%}$ & $27.5\%\,{\scriptstyle\pm\,1.0\%}$\\*
& {\ours} & $\mathbf{33.8\%}\,{\scriptstyle\pm\,1.5\%}$ & $\mathbf{63.2\%}\,{\scriptstyle\pm\,3.5\%}$ & $\mathbf{31.8\%}\,{\scriptstyle\pm\,1.2\%}$ & $\mathbf{31.8\%}\,{\scriptstyle\pm\,0.7\%}$\\
\midrule
\multirow{8}{*}{\textbf{\shortstack{FineAction\\-MIPL}}}
& \demipl & $36.0\%\,{\scriptstyle\pm\,1.8\%}$ & $48.8\%\,{\scriptstyle\pm\,2.8\%}$ & $18.3\%\,{\scriptstyle\pm\,1.3\%}$ & $44.6\%\,{\scriptstyle\pm\,2.1\%}$\\*
& \elimipl & $34.8\%\,{\scriptstyle\pm\,1.6\%}$ & $44.0\%\,{\scriptstyle\pm\,2.7\%}$ & $18.9\%\,{\scriptstyle\pm\,2.0\%}$ & $44.7\%\,{\scriptstyle\pm\,1.4\%}$\\*
& \miplma & $49.6\%\,{\scriptstyle\pm\,2.0\%}$ & $61.1\%\,{\scriptstyle\pm\,0.3\%}$ & $30.6\%\,{\scriptstyle\pm\,2.7\%}$ & $51.4\%\,{\scriptstyle\pm\,2.2\%}$\\*
& \promipl & $79.0\%\,{\scriptstyle\pm\,1.3\%}$ & $91.0\%\,{\scriptstyle\pm\,3.5\%}$ & $66.5\%\,{\scriptstyle\pm\,1.6\%}$ & $80.2\%\,{\scriptstyle\pm\,2.0\%}$\\*
& \psmipl & $57.2\%\,{\scriptstyle\pm\,1.1\%}$ & $63.1\%\,{\scriptstyle\pm\,1.6\%}$ & $39.4\%\,{\scriptstyle\pm\,1.8\%}$ & $29.2\%\,{\scriptstyle\pm\,14.1\%}$\\*
& \dualg & $53.7\%\,{\scriptstyle\pm\,0.6\%}$ & $66.3\%\,{\scriptstyle\pm\,1.2\%}$ & $40.3\%\,{\scriptstyle\pm\,1.1\%}$ & $57.4\%\,{\scriptstyle\pm\,0.7\%}$\\*
& \agopmipl & $\underline{81.7\%}\,{\scriptstyle\pm\,1.1\%}$ & $\mathbf{95.2\%}\,{\scriptstyle\pm\,1.2\%}$ & $\underline{67.5\%}\,{\scriptstyle\pm\,1.6\%}$ & $\underline{82.0\%}\,{\scriptstyle\pm\,3.0\%}$\\*
& {\ours} & $\mathbf{86.3\%}\,{\scriptstyle\pm\,1.5\%}$ & $\underline{92.0\%}\,{\scriptstyle\pm\,1.7\%}$ & $\mathbf{73.6\%}\,{\scriptstyle\pm\,1.1\%}$ & $\mathbf{85.3\%}\,{\scriptstyle\pm\,1.4\%}$\\
\bottomrule
\end{tabular}
\end{table}

\subsection{Efficiency Analysis} \label{Efficiency analysis}
We also evaluate the efficiency of \ours, presented in Figure~\ref{computation_acc}. Figure~\ref{computation_acc}  compares classification accuracy against training time. {\ours} exhibits a favorable accuracy--efficiency trade-off across most dataset--feature configurations; the exceptions occur in the VideoMAEv2 settings where {\agopmipl} can achieve a Pareto-superior operating point. 
\begin{figure}[!t]
\centering
\includegraphics[width=\linewidth]{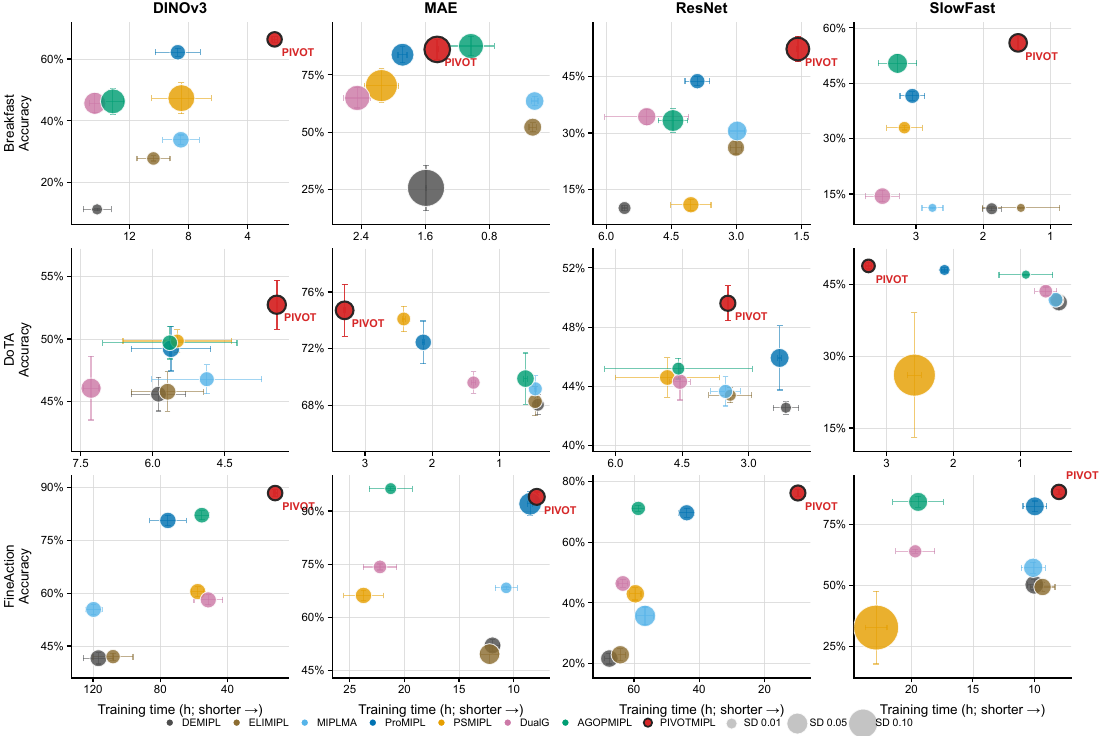}
\caption{Accuracy--training-time trade-off across the twelve dataset--feature configurations. Each panel corresponds to one dataset--feature setting and compares {\ours} with the MIPL baselines in terms of mean classification accuracy and total training time over five folds. Higher accuracy and shorter training time are preferred.}
\label{computation_acc}
\end{figure}

\subsection{Ablation Analysis}
\subsubsection{Ablation Summaries}\label{Ablation summaries}
Table~\ref{tab:ablation_summary} shows that the gains of {\ours} are distributed across multiple components rather than dominated by a single design choice. Replacing the plan-level objective with class-level KL (A3) or removing the temporal-marginal constraint from the teacher (A4) degrades performance in all 12 configurations, highlighting joint class--time supervision and dual-marginal projection as the most structurally consistent contributors. Removing the temporal encoder (A6) causes the largest average drop, while the strong performance gap retained against encoder-matched baselines indicates that temporal contextualization alone does not explain the full advantage. Finally, A1, A2, and A5 show that occupancy regularization, adaptive temporal allocation, and multiple prototypes each provide complementary gains, supporting the need to model both temporal concentration and intra-class variation. \looseness=-1
\begin{table}[!b]
\centering
\caption{Summary of the ablations across twelve dataset--feature configurations.}
\label{tab:ablation_summary}
\small
\begin{tabular}{llcc}
\toprule
Variant & Modification & Full-model wins & Mean gain\\
\midrule
A1 & No occupancy regularization & 10/12 & $+3.25\%$\\
A2 & Uniform temporal occupancy & 9/12 & $+0.72\%$\\
A3 & Class-level KL & 12/12 & $+3.26\%$\\
A4 & Class-marginal-only teacher & 12/12 & $+2.53\%$\\
A5 & Single prototype per class & 11/12 & $+2.52\%$\\
A6 & No temporal encoder & 11/12 & $+5.77\%$\\
\bottomrule
\end{tabular}
\end{table}

\subsubsection{Analysis Of The Effectiveness Of The Temporal Encoder} \label{te}
\begin{table}[!t]
\centering
\caption{Accuracy for \agopmipl, {\agopmipl} augmented with the same temporal encoder architecture as {\ours}, and {\ours}. The best and second-best values within each row are shown in bold and underlined, respectively.}
\label{1}
\small
\renewcommand{\arraystretch}{0.95}
\setlength{\tabcolsep}{8pt}
\begin{tabular}{llccc}
\toprule
Dataset & Feature & \agopmipl & \shortstack{{\agopmipl}\,$+$\\Temporal Encoder} & {\ours}\\
\midrule
\multirow{4}{*}{\textbf{Breakfast}} & DINOv3 & $46.4\%$ & $\underline{53.1\%}$ & $\mathbf{66.4\%}$\\
 & VideoMAEv2 & $\mathbf{79.3\%}$ & $\underline{78.1\%}$ & $77.4\%$\\
 & ResNet & $31.4\%$ & $\underline{45.4\%}$ & $\mathbf{52.2\%}$\\
 & SlowFast & $50.6\%$ & $\underline{53.1\%}$ & $\mathbf{53.9\%}$\\
\midrule
\multirow{4}{*}{\textbf{DoTA}} & DINOv3 & $\underline{49.2\%}$ & $48.4\%$ & $\mathbf{54.3\%}$\\
 & VideoMAEv2 & $69.5\%$ & $\underline{69.9\%}$ & $\mathbf{72.8\%}$\\
 & ResNet & $46.0\%$ & $\underline{48.6\%}$ & $\mathbf{50.9\%}$\\
 & SlowFast & $46.9\%$ & $\underline{47.6\%}$ & $\mathbf{48.4\%}$\\
\midrule
\multirow{4}{*}{\textbf{FineAction}} & DINOv3 & $85.5\%$ & $\underline{86.8\%}$ & $\mathbf{90.0\%}$\\
 & VideoMAEv2 & $\underline{96.7\%}$ & $\mathbf{97.2\%}$ & $92.4\%$\\
 & ResNet & $73.1\%$ & $\underline{75.4\%}$ & $\mathbf{78.1\%}$\\
 & SlowFast & $87.9\%$ & $\underline{89.2\%}$ & $\mathbf{90.0\%}$\\
\bottomrule
\end{tabular}
\end{table}

To examine whether adding temporal encoding alone can reproduce the performance advantage of {\ours}, we equip {\agopmipl}, the strongest baseline, with the same temporal encoder architecture. We compare the original {\agopmipl}, its encoder-augmented variant, and {\ours}. The result of the above shows {\ours} retains higher ACC in 10 configurations, with an average advantage of $2.83\%$ over the encoder-augmented baseline. Together with A6, this comparison suggests that temporal encoding contributes to performance, but adding the same encoder architecture to {\agopmipl} does not by itself account for the observed advantage of {\ours} across most configurations. \looseness=-1

\subsection{Parameter Sensitivity Analysis}  \label{sensitivity analysis}
We examine the sensitivity of {\ours} to the two transport-specific hyperparameters, the entropic coefficient $\varepsilon$ and the occupancy coefficient $\lambda$ in Eq.~(\ref{eq:occupancy_transport}), using three representative Fold-1 configurations: Breakfast--ResNet, DoTA--VideoMAEv2, and FineAction--SlowFast. As shown in Figure~\ref{Parameter Sensitivity}, performance varies only moderately over the evaluated ranges for both parameters, despite substantial differences in dataset scale and feature representation. The two parameters play complementary roles in the transport inference: $\varepsilon$ controls the sharpness of the class--prototype--time assignment, whereas $\lambda$ regulates the concentration of the temporal marginal. A very small $\varepsilon$ favors sharper matching, while increasing it produces smoother assignments; similarly, increasing $\lambda$ discourages excessive concentration of mass on a small number of temporal positions. The relatively broad performance plateaus therefore suggest that {\ours} does not depend on a narrowly tuned balance between these two effects. Instead, the structured teacher and plan-level supervision remain effective under a range of assignment smoothness and temporal occupancy regularization strengths. This robustness is particularly desirable in VCMIPL, where sequence length, temporal redundancy, and feature granularity vary substantially across datasets.

\begin{figure}[!t]
    \centering
    \includegraphics[width=\linewidth]{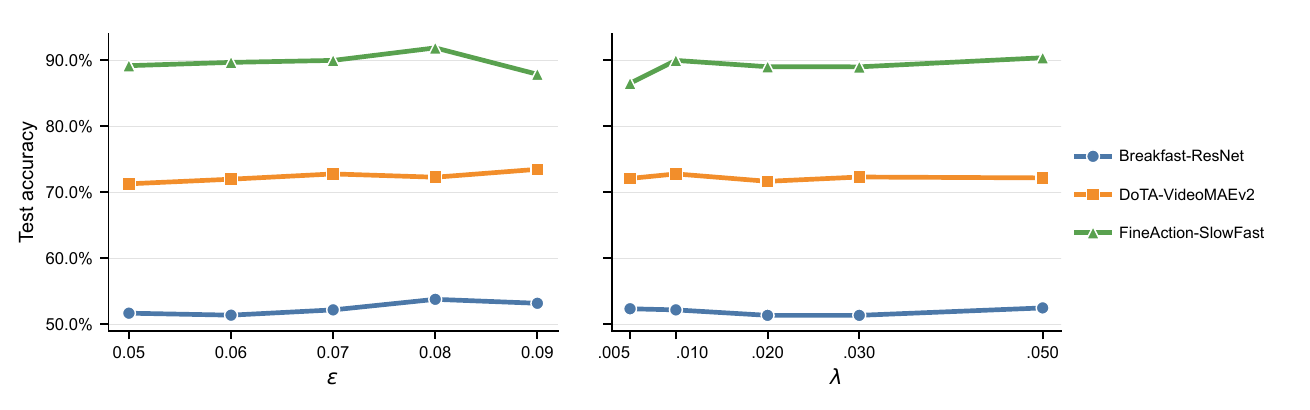}
    \vspace{-10pt}
    \caption{Sensitivity of {\ours} to the entropic coefficient $\varepsilon$ and occupancy coefficient $\lambda$ on three representative dataset--feature configurations.}
    \label{Parameter Sensitivity}
\end{figure}

\par\endgroup
\section{Limitations}
Our study focuses on offline VCMIPL within a fixed label space. Although the experiments span multiple datasets and feature representations, they do not establish robustness to distribution shifts or evolving class vocabularies. Assessing how class--time associations transfer across such changes remains a direction for future work. Moreover, the current formulation uses bidirectional temporal context and normalizes assignments over the observed video. Extending it to strictly online prediction would require causal contextualization and incremental assignment updates that account for incomplete temporal evidence.

\end{document}